\documentclass{article}

\usepackage[preprint,nonatbib]{neurips_2023}

\usepackage[utf8]{inputenc} 
\usepackage[T1]{fontenc}    
\usepackage{hyperref}       
\usepackage{url}            
\usepackage{booktabs}       
\usepackage{amsfonts}       
\usepackage{nicefrac}       
\usepackage{microtype}      
\usepackage{xcolor}         

\usepackage{bm}
\usepackage{amsmath,amssymb,amscd,amsthm}
\usepackage{mathtools}
\usepackage{algorithm}
\usepackage{algorithmic}
\usepackage{subfigmat}
\usepackage{wrapfig}

\theoremstyle{plain}
\newtheorem{theorem}{Theorem}

\theoremstyle{definition}
\newtheorem{definition}{Definition}
\newtheorem{assumption}{Assumption}

\theoremstyle{remark}
\newtheorem{proposition}{Proposition}
\newtheorem{remark}{Remark}
\newtheorem{exam}{Example}

\newcommand{\argmax}{\mathop{\rm arg~max}\limits}

\title{Prediction Algorithms Achieving Bayesian Decision Theoretical Optimality Based on Decision Trees as Data Observation Processes}

\author{%
  Yuta Nakahara\\
  Center for Data Science\\
  Waseda University\\
  Tokyo, Japan\\
  \texttt{y.nakahara@waseda.jp}\\
  \And
  Shota Saito\\
  Faculty of Informatics\\ 
  Gunma University\\
  Gunma, Japan\\
  \texttt{shota.s@gunma-u.ac.jp}\\
  \And
  Naoki Ichijo\\
  Dept. of Pure and Applied Math.\\
  Waseda University\\
  Tokyo, Japan\\
  \texttt{1jonao@fuji.waseda.jp} \\
  \And
  Koki Kazama\\
  Dept. of Pure and Applied Math.\\
  Waseda University\\
  Tokyo, Japan\\
  \texttt{kokikazama@aoni.waseda.jp}\\
  \And
  Toshiyasu Matsushima\\
  Dept. of Pure and Applied Math.\\
  Waseda University\\
  Tokyo, Japan\\
  \texttt{toshimat@waseda.jp}\\
}

\begin{document}

\maketitle

\begin{abstract}
  In the field of decision trees, most previous studies have difficulty ensuring the statistical optimality of a prediction of new data and suffer from overfitting because trees are usually used only to represent prediction functions to be constructed from given data. In contrast, some studies, including this paper, used the trees to represent stochastic data observation processes behind given data. Moreover, they derived the statistically optimal prediction, which is robust against overfitting, based on the Bayesian decision theory by assuming a prior distribution for the trees. However, these studies still have a problem in computing this Bayes optimal prediction because it involves an infeasible summation for all division patterns of a feature space, which is represented by the trees and some parameters. In particular, an open problem is a summation with respect to combinations of division axes, i.e., the assignment of features to inner nodes of the tree. We solve this by a Markov chain Monte Carlo method, whose step size is adaptively tuned according to a posterior distribution for the trees.
\end{abstract}

\section{Introduction}\label{intro}

Studies of decision trees have been primarily focused on the problem of overfitting, i.e., improvement of prediction accuracy for new data. It is because, while decision trees can be easily constructed to perfectly fit the training data by selecting appropriate features and increasing the tree's depth, such models often show poor performance when predicting new data. Therefore, various methods have been proposed to improve the accuracy of predicting new data, e.g., pruning (e.g., \cite{CART}), ensemble learning (e.g., \cite{RF, GradientBoost}), and introducing regularization terms in the cost function (e.g., \cite{XGBoost}). These methods have been successfully applied to various prediction problems on real-world data.

Nonetheless, it is challenging to prove the prediction made by these methods is statistically optimal or to discuss its room for improvement. In our opinion, a critical cause of this is that most prior studies have used trees solely to represent the prediction function (or hypotheses) to be constructed from given data and have not used to represent a stochastic data observation process behind the given data. Herein, we call such trees representing the predictive function \emph{function trees}. In principle, it is impossible to directly minimize a prediction error between the new data point and a predicted value without any assumption that both the training data and the new data point are observed according to a similar stochastic mechanism. More specifically, it is impossible to consider the optimal prediction based on the statistical decision theory (see, e.g., \cite{Berger}) without any assumption of stochastic data observation processes.  

In contrast, few studies utilized trees to represent stochastic data observation processes \cite{suko_alg, MTRF}. We call such trees representing the data observation processes \emph{model trees} herein. A model tree represents a division pattern of the feature space, i.e., the shape of the model tree represents the number of divisions and the features assigned to the inner nodes represent the division axes, in a similar manner to usual function trees. However, unlike the function trees, stochastic models are assigned to the leaf nodes of the model tree, and we assume the objective variables are observed according to them. Although the shape of the model tree, the features assigned to the inner nodes, and the parameters of the stochastic models on the leaf nodes are usually unobservable, assuming prior distributions on all these amounts and applying the Bayesian decision theory (see, e.g., \cite{Berger}), \cite{suko_alg, MTRF} provided a framework to directly minimize the expectation of the prediction error for new data instead of any cost function on the training data so that the overfitting could be avoided. Moreover, they theoretically derived the formula strictly minimizing that. The prediction made by this formula is called the Bayes optimal prediction.

However, this formula involves expectations for the tree's shape, the features on the inner nodes, and the parameters of stochastic models on the leaf nodes under their posterior distributions. Although an algorithm to efficiently and exactly calculate the expectations for the tree's shape and the parameters on the leaf nodes are proposed in \cite{suko_alg, MTRF}, the expectation for the features on the inner nodes is still an open problem. Although an approximative method was proposed in \cite{MTRF}, it loses the Bayes optimality.

Therefore, we solve this by a Markov chain Monte Carlo (MCMC) method (see, e.g., \cite{Bishop}) and propose an algorithm to predict new data. Our method retains the Bayes optimality after sufficient MCMC iterations. Therefore, it is the first prediction algorithm achieving the direct minimization of the expectation of the prediction error for new data instead of any cost function on the training data, which cannot be achieved by any function tree based approaches mentioned in the first paragraph. In our method, we adaptively tune the step size of the MCMC method according to the posterior distribution for trees. We confirm its effectiveness by numerical experiments. As a result, our method showed better prediction performance than some state-of-the-art methods \cite{XGBoost, lightGBM}.

As another related work, the first MCMC method for the model trees was reported by \cite{BayesianCART}. Although it seems not to be motivated by the Bayesian decision theoretical optimality but rather a stochastic search of the decision trees, it is able to be immediately applied to the Bayes optimal decision because it enables us to calculate the expectation for the model trees. However, this method approximates the expectations for both the features on the inner nodes and the shape of the tree by the Metropolis-Hastings (MH) method (see, e.g. \cite{Bishop}). In contrast, our method exactly calculates the expectation for the shape of the tree and approximates only the expectation for the features in the inner nodes. Therefore, our method is also regarded as an idea to accelerate or leverage the MCMC method of \cite{BayesianCART}. An effect from this perspective will be demonstrated in numerical experiments. \cite{BayesianCART} is further extended to a model called BART in \cite{BART}, in which data are observed according to a sum of multiple model trees. Therefore, our model will be extended in a similar manner. However, we focus on the single tree model in this paper.

\section{Preliminaries}

\subsection{Basic Notations}

Let the dimension of the continuous features be $p \in \mathbb{N} \coloneqq \{1,2, 3, \ldots \}$. Let the dimension of the binary features be $q \in \mathbb{N}$. Let $\bm x = (x_1, x_2, \dots , x_p, x_{p+1}, \dots , x_{p+q}) \in \mathbb{R}^p \times \{0,1\}^q \subset \mathbb{R}^{p+q}$ be an explanatory variable, where $x_1, \dots , x_p$ take continuous values and $x_{p+1}, \dots , x_{p+q}$ take binary values.
Also, $\mathcal{Y}$ denotes a set of possible values of objective variables.
Our discussion can be applied to both a discrete set (e.g., $\mathcal{Y}=\{0, 1\}$) and a continuous set (e.g., $\mathcal{Y}=\mathbb{R}$).
Let $Y$ be a random variable taking values in $\mathcal{Y}$ and $y \in \mathcal{Y}$ be a realization of $Y$.

\clearpage

\begin{wrapfigure}[13]{r}[0pt]{0.4\textwidth}
\centering
\includegraphics[width=\linewidth]{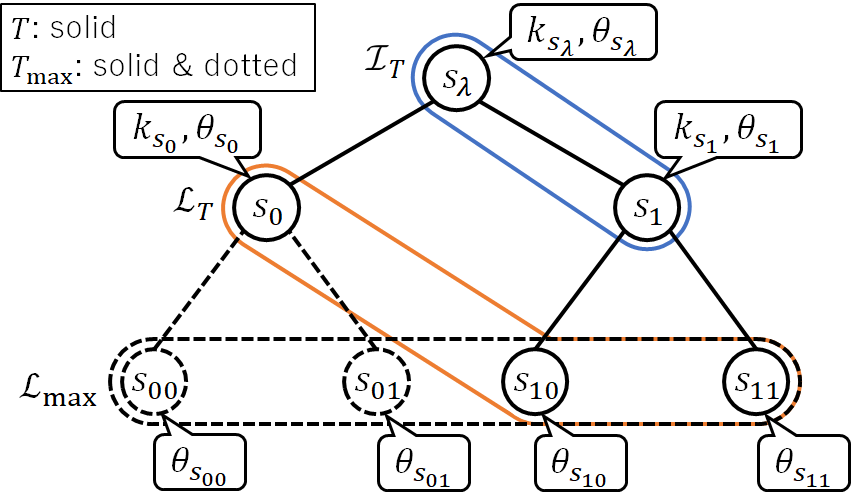}
\caption{The basic notations for a binary tree. Here, $D_\mathrm{max}$ is 2.}
\label{basic_notation}
\end{wrapfigure}

Regarding a tree, we use the following notations. See also Fig.\ \ref{basic_notation}.
Let $D_\mathrm{max} \in \mathbb{N}$ be the maximum depth of trees.
The perfect\footnote{All inner nodes have exactly two children and all leaf nodes have the same depth.} binary tree whose depth is $D_\mathrm{max}$ is denoted by $T_\mathrm{max}$.
The set of all nodes of $T_\mathrm{max}$ is denoted by $\mathcal{S}_\mathrm{max}$.
The set $\mathcal{S}_\mathrm{max}$ can be divided into two disjoint subsets: $\mathcal{L}_\mathrm{max} \subset \mathcal{S}_\mathrm{max}$ and $\mathcal{I}_\mathrm{max} \subset \mathcal{S}_\mathrm{max}$, where $\mathcal{L}_\mathrm{max}$ is the set of the leaf nodes of $T_\mathrm{max}$ and $\mathcal{I}_\mathrm{max}$ is the set of the inner nodes of $T_\mathrm{max}$.
In this paper, we consider a rooted tree, i.e., a tree that has a root node $s_\lambda \in \mathcal{S}_\mathrm{max}$.
Let $T$ be a full (also called proper) subtree of $T_\mathrm{max}$, where $T$'s root node is $s_\lambda$ and all inner nodes have exactly two children.
The set of all nodes of $T$ is denoted by $\mathcal{S}_T \subset \mathcal{S}_\mathrm{max}$. It can be divided into $\mathcal{L}_T \subset \mathcal{S}_T$ and $\mathcal{I}_T \subset \mathcal{S}_T$, where $\mathcal{L}_T$ is the set of the leaf nodes of $T$ and $\mathcal{I}_T$ is the set of the inner nodes of $T$.
The set of all full subtrees $T$ is denoted by $\mathcal{T}$.
As we will describe later in detail, a feature index $k_s \in \{ 1, 2, \dots , p+q \}$ is assigned to an inner node $s \in \mathcal{I}_\mathrm{max}$, and a feature assignment vector is denoted by $\bm k \coloneqq (k_s)_{s \in \mathcal{I}_\mathrm{max}} \in \mathcal{K} \coloneqq \{ 1, 2, \dots , p+q \}^{|\mathcal{I}_\mathrm{max}|}$.
Also as we will describe later in detail, a node $s \in \mathcal{S}_\mathrm{max}$ has a parameter $\theta_s$.
We use the notation $\bm \theta \coloneqq (\theta_s)_{s \in \mathcal{S}_\mathrm{max}}$.
The set of $\bm \theta$ is denoted by $\bm \Theta$.

\subsection{Stochastic Data Observation Process}

\begin{wrapfigure}[26]{r}[0pt]{0.5\textwidth}
\centering
\includegraphics[width=\linewidth]{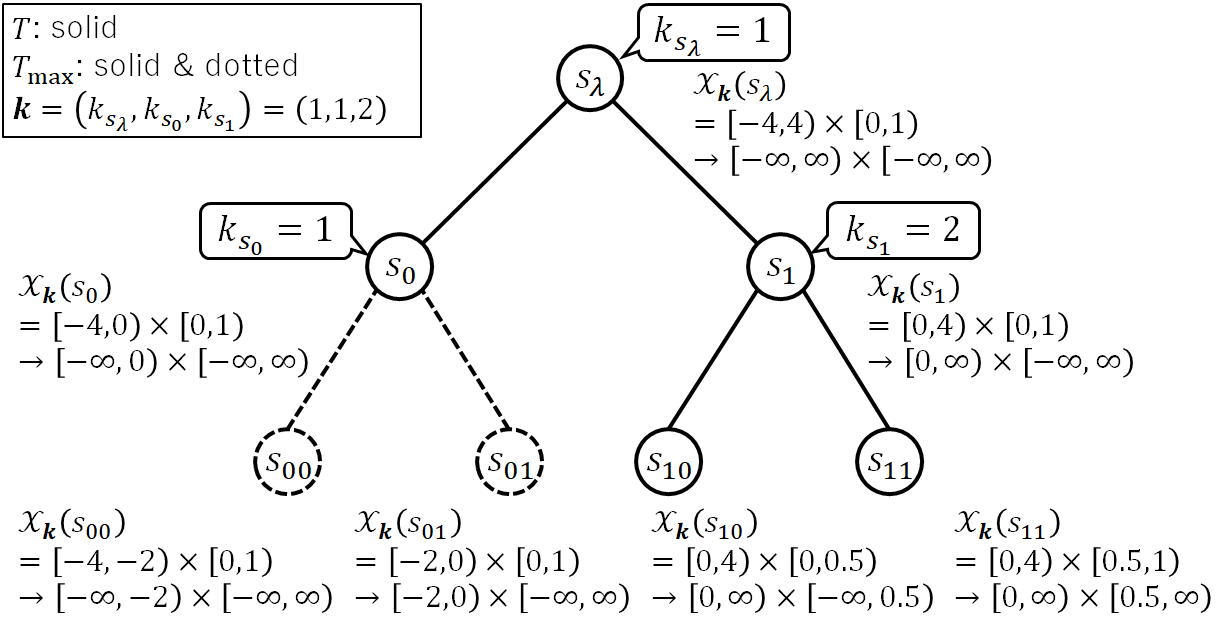}
\caption{An example of subspace division procedure. Here, $D_\mathrm{max} = 2$, $p = 1$, $q=1$, and $\bm k = (k_{s_\lambda}, k_{s_0}, k_{s_1}) = (1, 1, 2)$. First, the root node $s_\lambda$ has a subspace $\mathcal{X}_{\bm k}(s_\lambda) = [a_{1,s_\lambda},b_{1,s_\lambda}) \times [a_{2,s_\lambda}, b_{2,s_\lambda}) = [-4,4) \times [0,1)$. Next, $\mathcal{X}_{\bm k}(s_\lambda)$ is divided into $\mathcal{X}_{\bm k}(s_0)$ and $\mathcal{X}_{\bm k}(s_1)$. Its threshold is a midpoint of $a_{1,s_\lambda}$ and $b_{1,s_\lambda}$ because $k_{s_\lambda} = 1$. Similarly, $\mathcal{X}_{\bm k}(s_0)$ and $\mathcal{X}_{\bm k}(s_1)$ are divided into $\mathcal{X}_{\bm k}(s_{00})$, $\mathcal{X}_{\bm k}(s_{01})$, $\mathcal{X}_{\bm k}(s_{10})$, and $\mathcal{X}_{\bm k}(s_{11})$. After that, temporary minimum and maximum values are replaced with $-\infty$ and $\infty$, respectively. As a result, $\bigcup_{s \in \mathcal{L}_T} \mathcal{X}_{\bm k}(s) = \mathbb{R}^{p+q}$ holds for any $T \in \mathcal{T}$, e.g., if $T$ is a tree represented with solid lines, then $\mathcal{X}_{\bm k}(s_0) \cup \mathcal{X}_{\bm k}(s_{10}) \cup \mathcal{X}_{\bm k}(s_{11}) = \mathbb{R}^2$ since $\mathcal{L}_T = \{ s_0, s_{10}, s_{11} \}$.}
\label{generative_model}
\end{wrapfigure}

We assume the following probability distribution on an objective variable $y$ given an explanatory variable $\bm x$. Note that $\bm \theta$, $T$, and $\bm k$ are unobservable parameters and their posterior should be calculated later in a Bayesian manner. First, we define the following subspace division procedure and a leaf node corresponding to a given explanatory variable. Note that this procedure is not a tree construction method from given data but a definition of stochastic data observation process behind the given data.

\begin{definition}[$\mathcal{X}_{\bm k}(s)$ and $s_{\bm k, T}(\bm x)$]\label{subspace}
Given $\bm k$, let $\mathcal{X}_{\bm k}(s)$ for $s \in \mathcal{S}_\mathrm{max}$ denote a subspace of $\mathbb{R}^{p+q}$, which is recursively defined in the following manner (see also Fig.\ \ref{generative_model}).

First, for the root node $s_\lambda$, we assume 
\begin{align}
    \mathcal{X}_{\bm k} (s_\lambda) = [a_{1,s_\lambda},&b_{1,s_\lambda}) \times \cdots \nonumber \\
    & \times [a_{p+q,s_\lambda},b_{p+q,s_\lambda}).
\end{align}
Here, $a_{k,s_\lambda}, b_{k,s_\lambda} \in \mathbb{R}$ are just initial values to determine thresholds and they do not restrict the acceptable range of the features. At this point, we use temporary minimum and maximum values for continuous features and $a_{k,s_\lambda}=0$ and $b_{k,s_\lambda}=1$ for binary features.

Next, if the following holds for any inner node $s \in \mathcal{I}_\mathrm{max}$,
\begin{align}
    \mathcal{X}_{\bm k} (s) = [a_{1,s},b_{1,s}) \times \cdots \times [a_{p+q,s},b_{p+q,s}),
\end{align}
then the subspace assigned to the left child $s_l$ and the right child $s_r$ of $s$ is defined as follows, based on the feature index $k_s$ assigned to $s$.
\begin{align}
    &\mathcal{X}_{\bm k}(s_l) = \{ \bm x \in \mathcal{X}_{\bm k}(s) \mid a_{k_s, s} \leq x_{k_s} < (a_{k_s,s} + b_{k_s, s}) / 2 \}, \\ 
    &\mathcal{X}_{\bm k}(s_r) = \{ \bm x \in \mathcal{X}_{\bm k}(s) \mid (a_{k_s,s} + b_{k_s, s}) / 2 \leq x_{k_s} < b_{k_s, s} \}.
\end{align}
In other words, the threshold is deterministically placed at a midpoint of the assigned subspace.


Lastly, we replace $a_{k,s}$ with $-\infty$ for any $k \in \{ 1, \dots , p+q \}$ and $s \in \mathcal{S}_\mathrm{max}$ such that $a_{k,s} = a_{k,s_\lambda}$. Similarly, we replace $b_{k,s}$ with $\infty$ for any $k \in \{ 1, \dots , p+q \}$ and $s \in \mathcal{S}_\mathrm{max}$ such that $b_{k,s} = b_{k,s_\lambda}$.

By this procedure, each $s \in \mathcal{S}_\mathrm{max}$ is assigned to a subspace of $\mathbb{R}^{p+q}$ and the following holds: for any $T \in \mathcal{T}$, $\bigcup_{s \in \mathcal{L}_T} \mathcal{X}_{\bm k}(s) = \mathbb{R}^{p+q}$, and for any $s, s' \in \mathcal{L}_T$, $s \neq s' \Rightarrow \mathcal{X}_{\bm k}(s) \cap \mathcal{X}_{\bm k}(s') = \emptyset$. Therefore, given $\bm k$ and $T$, we can uniquely determine a node $s \in \mathcal{L}_T$ such that $\bm x \in \mathcal{X}_{\bm k}(s)$, for any $\bm x \in \mathbb{R}^{p+q}$. Let $s_{\bm k, T}(\bm x)$ represents this node.
\end{definition}

Using the above notation, we impose the following assumptions on the probability distribution of an objective variable $y$ given an explanatory variable $\bm x$.

\begin{assumption} \label{Distribution_y}
Given $\bm k$ and $T$, let $s_{\bm k, T} (\bm x) \in \mathcal{L}_T$ denote the leaf node defined in Def.\ \ref{subspace}, which is uniquely and deterministically obtained from the explanatory variable $\bm x$. Then, we assume
\begin{align}
p(y | \bm x, \bm \theta, T, \bm k) = p(y | \theta_{s_{\bm k, T}(\bm x)}).\label{eq_gen_model}
\end{align}
That is, we assume that $y$ is independent of any other parameter than that assigned to $s_{\bm k, T}(\bm x)$.
\end{assumption}

\begin{assumption} \label{Distribution_theta}
We assume the prior distribution on $\bm \theta$ has the following form: $p(\bm \theta) = \prod_{s \in \mathcal{S}_\mathrm{max}} p(\theta_s)$.
In addition, we assume each prior $p(\theta_s)$ is a conjugate prior for $p(y | \theta_s)$ and we can calculate its predictive distribution $p(y) = \int p(y | \theta_s) p(\theta_s) \mathrm{d}\theta_s$ with an acceptable cost.
\end{assumption}

The following examples fulfill the above assumptions.

\begin{exam} \label{Example_y}
For example, when $\mathcal{Y}$ is finite, we can assume the categorical distribution $\mathrm{Cat}(y | \bm \pi_s)$ and the Dirichlet prior $\mathrm{Dir}(\bm \pi_s | \bm \alpha )$.
When $y$ is a count data, i.e., $\mathcal{Y} = \{0, 1, \dots \}$, we can assume the Poisson distribution $\mathrm{Po}(y | \nu_s)$ and the gamma prior $\mathrm{Gam}(\nu_s | \alpha, \beta)$.
When $\mathcal{Y}$ is continuous, we can assume the normal distribution $\mathcal{N}(y | \mu_s, \sigma_s^2)$ and the normal-gamma prior $\mathcal{N}(\mu_s | m, \gamma \sigma_s^2) \mathrm{Gam}(1 / \sigma_s^2 | \alpha, \beta)$. Further, we can also assume a more complicated model, e.g., linear regression (LR) model $\mathcal{N}(y | \bm w_s^\top \bm x, \sigma_s^2)$ and the normal-gamma prior $\mathcal{N}(\bm w_s | \bm m, \bm \Lambda / \sigma_s^2) \mathrm{Gam}(1 / \sigma_s^2 | \alpha, \beta)$, as long as it satisfies Assumption \ref{Distribution_theta}. This flexibility is one of advantages of our model.
\end{exam}

\begin{figure*}[tbp]
\centering
\includegraphics[width=0.9\linewidth]{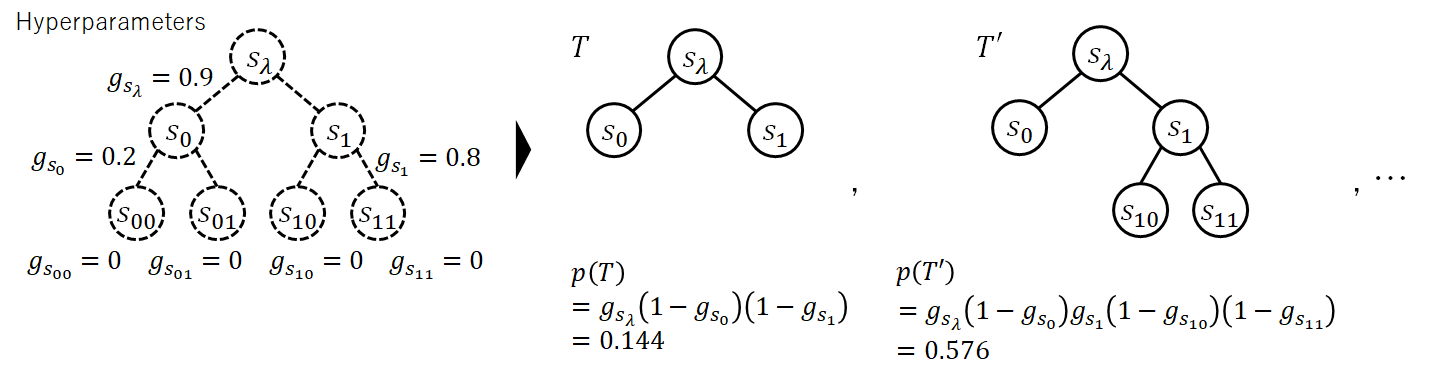}
\caption{An example of the prior distribution on $T \in \mathcal{T}$. Its hyperparameters are given as shown in the left. $s_\lambda$ becomes an inner node with probability $g_{s_\lambda} = 0.9$. Therefore, the data observation process includes $x_{k_{s_\lambda}}$ with probability $g_{s_\lambda} = 0.9$. Similarly, $s_1$ becomes an inner node and the data observation process includes $x_{k_{s_1}}$ with probability $g_{s_\lambda}g_{s_1} = 0.9 \cdot 0.8$ (see also Remark 2 of \cite{full_rooted_trees}).}
\label{tree_distribution}
\end{figure*}

We assume the following prior distribution of $T \in \mathcal{T}$, which has been used in \cite{CT, suko_alg, MTRF, full_rooted_trees}.
\begin{assumption} \label{Distribution_T}
Given $D_\mathrm{max}$, we assume the following probability distribution on the set $\mathcal{T}$ of full trees $T$, which are subtrees of the perfect binary tree $T_\mathrm{max}$ whose depth is $D_\mathrm{max}$:
\begin{align}
p(T) \coloneqq \prod_{s \in \mathcal{I}_T} g_s \prod_{s' \in \mathcal{L}_T} (1-g_{s'}), \label{DistributionTree}
\end{align}
where $g_s \in [0, 1]$ is a given hyperparameter representing an edge spreading probability of a node $s$. For $s \in \mathcal{L}_\mathrm{max}$, we assume $g_s = 0$.
\end{assumption}
Properties of this distribution are discussed in \cite{full_rooted_trees}, e.g., Eq.\ \eqref{DistributionTree} satisfies $\sum_{T \in \mathcal{T}} p(T) = 1$.

\begin{exam}
Figure \ref{tree_distribution} shows an example of $p(T)$. The hyperparameter $g_s$ represents the edge spreading probability under the condition that all the ancestor nodes of $s$ extend their edges. In other words, the data observation process includes the explanatory variable $x_{k_s}$ with the prior probability $g_s$ under the condition that it includes all the explanatory variables assigned to the ancestor nodes of $s$ (see also Remark 2 of \cite{full_rooted_trees}). Therefore, the prior probability that an explanatory variable on a node is included in the model decreases exponentially with its depth. In the learning phase, this property of the prior distribution prevents overfitting.
\end{exam}

Lastly, the prior distribution of $\bm k$ is as follows.
\begin{assumption} \label{Distribution_k}
We assume that $k_s$ is independently assigned to each $s \in \mathcal{I}_\mathrm{max}$ with probability $1/(p+q)$, that is $p(\bm k )$ is the uniform distribution on $\mathcal{K}$.
\end{assumption}
In other words, each feature can be assigned multiple times on a path from the root node to a leaf node. We assume this for simplicity, and we can also restrict the number of feature assignments.

\subsection{Problem Setup} \label{setup}

We deal with a prediction problem of a new objective variable $y_{n+1} \in \mathcal{Y}$ corresponding to an explanatory variable $\bm x_{n+1} \in \mathbb{R}^p \times \{0, 1\}^q$ from given training data $(\bm x^n, y^n) \coloneqq \{ (\bm x_i, y_i) \}_{i \in \{ 1, 2, \dots n \}} \in (\mathbb{R}^p \times \{0, 1\}^q \times \mathcal{Y})^n$, where $n \in \mathbb{N}$ is the sample size and we assume $y_i$ independently follows \eqref{eq_gen_model} given $\bm x_i$. We assume we know the maximum depth $D_\mathrm{max}$ of the trees, the initial range $[a_{k,s_\lambda}, b_{k,s_\lambda})$ of the subspace for any $k \in \mathcal{K}$ (see Definition \ref{subspace}), and the hyperparameters $(g_s)_{s \in \mathcal{S}_\mathrm{max}}$ of $p(T)$ (see Assumption \ref{Distribution_T}), while $\bm \theta$, $T$, and $\bm k$ are unknown. We also make Assumptions \ref{Distribution_y}, \ref{Distribution_theta}, \ref{Distribution_T}, and \ref{Distribution_k}.

\subsection{Bayes Optimal Prediction for New Data Point}
We formulate the optimal prediction under the Bayesian decision theory (see, e.g., \cite{Berger}).
As we have described, we would like to predict the value of the new objective variable $y_{n+1}$ corresponding to $\bm x_{n+1}$ given training data $(\bm x^n, y^n)$.
Hence, the decision function $\delta$, which outputs a predicted value, is defined as $\delta: \mathbb{R}^p \times \{0, 1\}^q \times (\mathbb{R}^p \times \{0, 1\}^q \times \mathcal{Y})^n \to \mathcal{Y}$, and the Bayes risk function $\mathrm{BR}(\delta)$ based on the 0-1 loss $\ell_{0-1}(\delta (\bm x_{n+1}, \bm x^n, y^n), y_{n+1})$ is defined as follows:
\begin{align}
\mathrm{BR}(\delta) \coloneqq \sum_{\bm k \in \mathcal{K}} \sum_{T \in \mathcal{T}} \int_{\bm \Theta} p(\bm k, T, \bm \theta) \int_{\mathcal{Y}^n} p(y^n | & \bm x^n, \bm \theta, T, \bm k) \int_{\mathcal{Y}} p(y_{n+1} | \bm x_{n+1}, \bm \theta, T, \bm k) \nonumber \\
&\times \ell_{0-1}(\delta (\bm x_{n+1}, \bm x^n, y^n), y_{n+1}) \mathrm{d}y_{n+1} \mathrm{d}y^n \mathrm{d}\bm \theta. \label{BayesRisk}
\end{align}
When $\mathcal{Y}$ is a finite set, the integral with respect to $y$ in \eqref{BayesRisk} is replaced by the summation.
Note that we can also assume other usual loss functions in the Bayes decision theory, e.g., the squared loss.\footnote{Herein, we regard the explanatory variables $\bm x^n$ and $\bm x_{n+1}$ are given constants. We can also regard them as random variables. In such a case, an additional expectation for $p(\bm x^n, \bm x_{n+1})$ is required to define $\mathrm{BR}(\delta)$. However, the Bayes optimal decision $\delta^*$ will be the same as \eqref{BayesOptimalPrediction}.}

The Bayes risk function $\mathrm{BR}(\delta)$ is an evaluation criterion of $\delta$, and
it is known that the optimal decision $\delta^*$ that minimizes $\mathrm{BR}(\delta)$ is given as follows.
\begin{proposition}[\cite{MTRF}] \label{OptimalPrediction}
The optimal decision $\delta^* (\bm x_{n+1}, \bm x^n, y^n)$ that minimizes \eqref{BayesRisk} is
\begin{align}
&\delta^* (\bm x_{n+1}, \bm x^n, y^n) = \argmax_{y_{n+1} \in \mathcal{Y}} \sum_{\bm k \in \mathcal{K}} \sum_{T \in \mathcal{T}} \int_{\bm \Theta} p(y_{n+1} | \bm x_{n+1}, \bm \theta, T, \bm k) p(\bm \theta, T, \bm k | \bm x^n, y^n) \mathrm{d} \bm \theta. \label{BayesOptimalPrediction}
\end{align}
\end{proposition}
In this paper, we call $\delta^* (\bm x_{n+1}, \bm x^n, y^n)$ the {\it Bayes optimal prediction}.
For the readers not familiar with Bayesian decision theory, the proof of Proposition \ref{OptimalPrediction} is given in the supplementary materials. For more detail, see e.g., \cite{Berger}.

In order to see the calculation of \eqref{BayesOptimalPrediction}, we decompose it into three components as follows:
\begin{align}
&q (y_{n+1} | \bm x_{n+1}, \bm x^n, y^n, T, \bm k) \coloneqq \int_{\bm \Theta} p(y_{n+1} | \bm x_{n+1}, \bm \theta, T, \bm k) p(\bm \theta | \bm x^n, y^n, T, \bm k) \mathrm{d}\bm \theta, \label{q} \\
&\tilde{q} (y_{n+1} | \bm x_{n+1}, \bm x^n, y^n, \bm k) \coloneqq \sum_{T \in \mathcal{T}} p(T | \bm x^n, y^n, \bm k) q(y_{n+1} | \bm x_{n+1}, \bm x^n, y^n, T, \bm k), \label{tildeq} \\
&\tilde{\tilde{q}} (y_{n+1} | \bm x_{n+1}, \bm x^n, y^n) \coloneqq \sum_{\bm k \in \mathcal{K}} p(\bm k | \bm x^n, y^n) \tilde{q}(y_{n+1} | \bm x_{n+1}, \bm x^n, y^n, \bm k). \label{tildetildeq}
\end{align}
By using these notations, \eqref{BayesOptimalPrediction} is rewritten as follows:
\begin{align}
\delta^* \! (\bm x_{n+1}, \bm x^n\!, y^n) = \argmax_{y_{n+1} \in \mathcal{Y}} \tilde{\tilde{q}} (y_{n+1} | \bm x_{n+1}, \bm x^n\!, y^n). \label{BayesOptimalWithQtildetilde}
\end{align}

We can efficiently calculate \eqref{q} under Assumption \ref{Distribution_theta}.
To calculate \eqref{tildeq}, \cite{suko_alg} and \cite{MTRF} introduced the following notion called meta-tree.
\begin{align}
M_{T,\bm k} \coloneqq \{ (\bm k', T') & \in \mathcal{K} \times \mathcal{T} \mid \bm k' = \bm k \text{ and $T'$ is a subtree of $T$} \}.
\end{align}
Its hierarchical structure enables calculating the summation of $p(T | \bm x^n, y^n, \bm k)q(y_{n+1} | \bm x_{n+1}, \bm x^n,$ $y^n, T, \bm k)$ in the meta-tree $M_{T,\bm k}$ under Assumption \ref{Distribution_T}. If $T$ of the meta-tree $M_{T, \bm k}$ equals $T_\mathrm{max}$, such summation is equivalent to calculating \eqref{tildeq}. Moreover, its computational cost is only $O(D_\mathrm{max} n)$ and retains the Bayes optimality. 

Therefore, we focus on the efficient calculation of \eqref{tildetildeq}, that is, the summation with respect to the meta-trees $M_{T_\mathrm{max}, \bm k}$, which has not been established yet.
Although an approximative method to calculate \eqref{tildetildeq} has been proposed in \cite{MTRF}, it loses the Bayes optimality.

Note that we do not learn the thresholds for subspace partitioning because they are deterministiclly derived from $\bm k$ in our setup (see Definition \ref{subspace}). In other words, we regard the problem of threshold learning as the problem of learning how many times the same $k$ is assigned on a path from the root node to a leaf node, and optimally solve it in Bayesian manner.

\section{Meta-Tree Markov Chain Monte Carlo Methods}\label{SectionMTMCMC}

This section describes our main results. We propose a Markov chain Monte Carlo (MCMC) method to calculate \eqref{tildetildeq} and construct an algorithm to predict new data, i.e., we approximate \eqref{BayesOptimalWithQtildetilde} as follows:
\begin{align}
\delta^* (\bm x_{n+1}, \bm x^n, y^n) \approx \argmax_{y_{n+1} \in \mathcal{Y}} \frac{1}{t_\mathrm{end}} \sum_{t=1}^{t_\mathrm{end}} \tilde{q}(y_{n+1} | \bm x_{n+1}, \bm x^n, y^n, \bm k^{(t)}), \label{ApproximatedBayesOptimal}
\end{align}
where $t_\mathrm{end} \in \mathbb{N}$ is the maximum number of the MCMC iteration and $\{ \bm k^{(t)} \}_{t=1}^{t_\mathrm{end}}$ denotes a sample following $p(\bm k | \bm x^n, y^n)$, which is obtained by our MCMC method. We call this method {\it meta-tree Markov chain Monte Carlo (MTMCMC) method}. Specifically, we propose a MH algorithm (see, e.g., \cite{Bishop}) and extend it to a replica exchange Monte Carlo (REMC) method (e.g., \cite{ReplicaExchange}) to deal with multimodality of the posterior distribution. Herein, we only describe the underlying MH method. The extension to REMC method is described in supplementary materials.

As usual MH methods, we generate $\bm k^*$ from a proposal distribution $q(\bm k^* | \bm k^{(t-1)})$ in the $t$th iteration of our algorithm. Then, it will be accepted according to the following acceptance probability.
\begin{align}
A (\bm k^*, \bm k^{(t-1)}) = \min \left\{ 1, \frac{p(\bm k^* | \bm x^n, y^n) q(\bm k^{(t-1)} | \bm k^*)}{p(\bm k^{(t-1)} | \bm x^n, y^n) q(\bm k^* | \bm k^{(t-1)})} \right\}. \label{AcceptanceProbability1}
\end{align}
If $\bm k^*$ is accepted, we make $\bm k^{(t)} \leftarrow \bm k^*$, otherwise $\bm k^{(t)} \leftarrow \bm k^{(t-1)}$. 

\begin{proposition}\label{MCMCValidity}
If $q(\bm k^* | \bm k^{(t-1)})$ is time-invariant and $q(\bm k^* | \bm k^{(t-1)}) > 0$ holds for any $\bm k^*$ and $\bm k^{(t-1)}$ through this process, then the \textit{detailed balance} is satisfied and an empirical distribution of the obtained sample $\{ \bm k^{(t)} \}_{t=1}^{t_\mathrm{end}}$ converges to the objective distribution $p(\bm k | \bm x^n, y^n )$ after sufficient iteration.
\end{proposition}
For the readers not familiar with MCMC method, we briefly prove this proposition in the supplementary materials. For more detail, see e.g., Chapter 11 of \cite{Bishop}

In our case, from the Bayes' theorem and Assumption \ref{Distribution_k}, Eq.\ \eqref{AcceptanceProbability1} is further transformed as follows:
\begin{align}
A (\bm k^*, \bm k^{(t-1)}) = \min \left\{ 1, \frac{p(y^n | \bm x^n, \bm k^*) q(\bm k^{(t-1)} | \bm k^*)}{p(y^n | \bm x^n, \bm k^{(t-1)}) q(\bm k^* | \bm k^{(t-1)})} \right\}. \label{acceptance}
\end{align}
In general, $p(y^n | \bm x^n, \bm k^*)$ and $p(y^n | \bm x^n, \bm k^{(t-1)})$ in \eqref{acceptance} cannot be efficiently calculated since it requires marginalization for $T$ and $\bm \theta$. However, we can calculate them by an algorithm proposed in \cite{batch_updating}. Therefore, at this point, the rest of the problem is design of the proposal distribution $q(\bm k^* | \bm k^{(t-1)})$.

\subsection{Design of Proposal Distribution}

Asymptotically, we can use any time-invariant distribution that satisfies $q(\bm k^* | \bm k^{(t-1)}) > 0$ for any $\bm k^*$ and $\bm k^{(t-1)}$, e.g., the uniform distribution $q(\bm k^* | \bm k^{(t-1)}) = (p+q)^{-|\mathcal{I}_\mathrm{max}|}$. However, its design crucially affects the performance on a finite MCMC sample.
This can be explained from a viewpoint of an analogy of the MH algorithm and a neighborhood searching algorithm. In the MH algorithm, $\bm k^*$ is proposed from a kind of neighborhood of $\bm k^{(t-1)}$ according to $q(\bm k^* | \bm k^{(t-1)})$. Roughly speaking, it will be accepted if it increases the probability of the objective distribution, i.e., $p(\bm k^* | \bm x^n, y^n) > p(\bm k^{(t-1)} | \bm x^n, y^n)$. Since the entropy of $q(\bm k^* | \bm k^{(t-1)})$ corresponds to the step size of neighborhood search, it should be larger to accelerate the search but it should be smaller to increase the acceptance ratio. Therefore, a desirable $q(\bm k^* | \bm k^{(t-1)})$ should induce many changes in the elements of $\bm k^{(t-1)}$ when $p(\bm k^{(t-1)} | \bm x^n, y^n)$ is small and a few changes when $p(\bm k^{(t-1)} | \bm x^n, y^n)$ is large. Note that $\bm k$ is discrete and hierarchically structured. Therefore, we cannot use the derivative of $p(\bm k^{(t-1)} | \bm x^n, y^n)$, and any Gibbs sampler for our model has not been reported to our best knowledge. Then, we use the posterior distribution $p(T | \bm x^n, y^n, \bm k^{(t-1)})$ as a heuristic to design $q(\bm k^* | \bm k^{(t-1)})$.

First, we have the following proposition.
\begin{proposition}[\cite{suko_alg, MTRF, full_rooted_trees}]
For any $\bm x^n$, $y^n$, and $\bm k$, the posterior distribution of $T$ is represented as follows:
\begin{align}
p(T | \bm x^n, y^n, \bm k) = \prod_{s \in \mathcal{I}_T} g_{s|\bm x^n, y^n, \bm k} \prod_{s' \in \mathcal{L}_T} (1-g_{s'|\bm x^n, y^n, \bm k}),
\end{align}
where $g_{s|\bm x^n, y^n, \bm k} \in [0,1]$ is a posterior parameter calculated from $\bm x^n$, $y^n$, and $\bm k$ for each $s \in \mathcal{S}_\mathrm{max}$.
\end{proposition}

\begin{wrapfigure}[12]{r}[0pt]{0.4\textwidth}
\centering
\includegraphics[width=\linewidth]{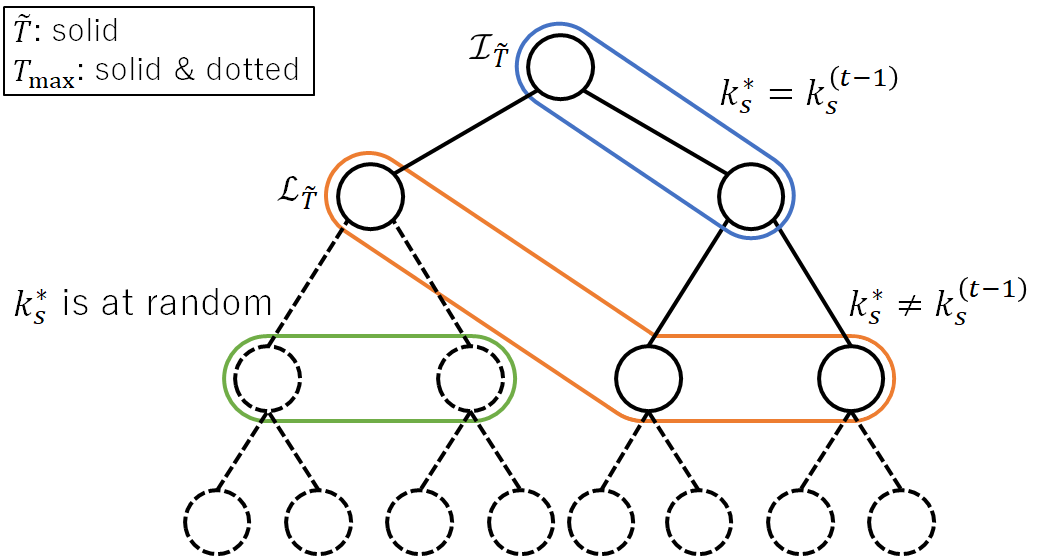}
\caption{An example of $\tilde{T}$ and properties of $\bm k^*$. Here, we assume $D_\mathrm{max}=3$. 
}
\label{proposal_distribution}
\end{wrapfigure}

Therefore, in the $t$th iteration, we can represent the posterior distribution $p(T | \bm x^n, y^n, \bm k^{(t-1)})$ in the same form as the prior distribution $p(T)$ with a posterior edge spreading probability $g_{s | \bm x^n, y^n, \bm k^{(t-1)}}$ of a node $s$. In other words, given $\bm k^{(t-1)}$, the data observation process includes an explanatory variable $x_{k_s^{(t-1)}}$ assigned to $s$ with a posterior probability $g_{s | \bm x^n, y^n, \bm k^{(t-1)}}$ under the condition that it includes the explanatory variables assigned to all the ancestor nodes of $s$ (see also Fig.\ \ref{tree_distribution}). We use this probability $g_{s|\bm x^n, y^n, \bm k^{(t-1)}}$ as a heuristic to determine the fixed elements of $\bm k^*$, that is, the smaller $g_{s|\bm x^n, y^n, \bm k^{(t-1)}}$ the node $s$ has, the more frequently $k_s^{(t-1)}$ is changed. Consequently, we generate $\bm k^*$ according to the following procedure, see also Fig.\ \ref{proposal_distribution}. (The initial value $\bm k^{(0)}$ is generated from the uniform distribution on $\mathcal{K}$.)

\textbf{1.} $\tilde{T}$ is generated according to
\begin{align}
q(\tilde{T} | & \bm x^n, y^n, \bm k^{(t-1)}) \coloneqq \prod_{s \in \mathcal{I}_{\tilde{T}}} \min \{ g_{s|\bm x^n, y^n, \bm k^{(t-1)}}, \bar{g} \} \prod_{s' \in \mathcal{L}_{\tilde{T}}} (1-\min \{ g_{s|\bm x^n, y^n, \bm k^{(t-1)}}, \bar{g} \}), \label{Probability of T^(t)}
\end{align}
where $\bar{g} \in [0,1]$ is predetermined in a burn-in phase (see also the supplementary materials).

\textbf{2.} For $s \in \mathcal{I}_{\tilde{T}}$, $k_s^{(t-1)}$ is fixed and $k_s^* = k_s^{(t-1)}$.

\textbf{3.} For $s \in \mathcal{L}_{\tilde{T}} \cap \mathcal{I}_\mathrm{max}$, $k_s^{(t-1)}$ is changed according to the uniform distribution on $\{ 1, 2, \dots , p+q \} \setminus \{ k_s^{(t-1)} \}$.

\textbf{4.} The others are changed according to the uniform distribution on $\{ 1, 2, \dots , p+q \}$.

Note that $\tilde{T}$ is uniquely determined from $\bm k^*$ and $\bm k^{(t-1)}$ as the maximum tree that satisfies $k^*_s = k^{(t-1)}_s$ for all $s \in \mathcal{I}_{\tilde{T}}$. Therefore, the proposal distribution $q(\bm k^* | \bm k^{(t-1)})$ is represented as follows:
\begin{align}
&q(\bm k^* | \bm k^{(t-1)}) = q(\tilde{T} | \bm x^n, y^n, \bm k^{(t-1)}) (p+q-1)^{-|\mathcal{L}_{\tilde{T}} \cap \mathcal{I}_\mathrm{max}|} (p+q)^{-|\mathcal{I}_\mathrm{max} \backslash \mathcal{S}_{\tilde{T}}|}. \label{TruncatedProposal}
\end{align}

Moreover, the following theorem holds.
\begin{theorem}
    By using the MCMC sample obtained by the MH method based on the proposal distribution \eqref{TruncatedProposal} and the acceptance probability \eqref{acceptance}, the MTMCMC method defined in \eqref{ApproximatedBayesOptimal} minimizes the Bayes risk function \eqref{BayesRisk}, i.e., achieves the Bayes optimality, after sufficient MCMC iterations.
\end{theorem}

\begin{proof}
If $\bm k^{(t-1)} = \bm k^{(t'-1)}$ holds, then $q(\bm k^* | \bm k^{(t-1)}) = q(\bm k^* | \bm k^{(t'-1)})$ clearly holds for any $\bm k^*$ even when $t \neq t'$. Therefore, a Markov chain of $\bm k^{(t)}$ induced from $q(\bm k^* | \bm k^{(t-1)})$ and $A(\bm k^*, \bm k^{(t-1)})$ is time-invariant. Moreover, $q(\bm k^* | \bm k^{(t-1)}) > 0$ holds for any $\bm k^*$ and $\bm k^{(t-1)}$. Therefore, the induced Markov chain of $\bm k^{(t)}$ is ergodic. Then, $q(\bm k^* | \bm k^{(t-1)})$ satisfies the condition of Proposition \ref{MCMCValidity}. Therefore, empirical distribution of the obtained sample converges to $p(\bm k | \bm x^n, y^n)$. Lastly, the right-hand side of \eqref{ApproximatedBayesOptimal} converges to the left-hand side, which is the decision function strictly minimizing the Bayes risk function, after sufficient MCMC iterations because of the law of large numbers.
\end{proof}

\begin{remark}
$\bar{g}$ is an additional parameter to control the entropy of $q(\bm k^* | \bm k^{(t-1)})$. When $n$ is large, $g_{s|\bm x^n, y^n, \bm k^{(t-1)}}$ for the nodes near the root $s_\lambda$ numerically equals to 1. Then, $k_s$ for them tends to be fixed and ergodicity will be collapsed. Introducing $\bar{g}$, all the elements of $\bm k^{(t-1)}$ are refreshed with the probability $1-\bar{g}$ and the ergodicity is ensured. This induces a ``jump'' of $\bm k^*$ and has some effects to deal with multimodality of the posterior distribution. A more effective approach to multimodality is extending our MH method to the REMC method, which is described in the supplementary materials.
\end{remark}

\begin{remark}\label{cancel}
Because of the uniqueness of $\tilde{T}$, transition from $\bm k^{(t-1)}$ to $\bm k^*$ cannot occur through any other tree than $\tilde{T}$, and vice versa. Therefore, $q(\bm k^{(t-1)} | \bm k^*)$ is represented as follows.
\begin{align}
&q(\bm k^{(t-1)} | \bm k^*) = q(\tilde{T} | \bm x^n, y^n, \bm k^*) (p+q-1)^{-|\mathcal{L}_{\tilde{T}} \cap \mathcal{I}_\mathrm{max}|} (p+q)^{-|\mathcal{I}_\mathrm{max} \backslash \mathcal{S}_{\tilde{T}}|}, \label{TruncatedProposal2}
\end{align}
where $\tilde{T}$ is same as that in \eqref{TruncatedProposal}. As a result, we can efficiently evaluate \eqref{acceptance} because many terms in the numerator and the denominator of \eqref{acceptance} are canceled by substituting \eqref{TruncatedProposal} and \eqref{TruncatedProposal2}. Further complexity reduction and complexity analysis are described in the supplementary materials.
\end{remark}

\section{Experiments}

Herein, we introduce only two experiments. In the supplementary materials, we described the others, e.g., confirmation of convergence of the approximated posterior to the true posterior; comparison with the uniform proposal distribution and the other tree posterior based proposal distribution; and confirmation of behavior of likelihood during the MCMC sampling.

First, we summarize the methods used in this section and their abbreviations. Most methods are used with their default hyperparameters (see the supplementary materials for detail).

\textbf{MTMCMC-Be-100(50) etc.:} the method proposed in this paper. The letters next to MTMCMC mean a stochastic model of $y$ assigned to each leaf node (see also Example \ref{Example_y}). Be, Po, and LR means the Bernoulli distribution, the Poisson distribution, and the LR model, respectively. The numbers at the end mean the number of MCMC iterations and the length of burn-in, e.g., 100(50) means we make 150 proposals and remove the first 50 of them. \textbf{MTRF-Be etc.:} the meta-tree random forest \cite{MTRF} implemented in \cite{bayesml}. The letters next to MTRF have a similar meaning to MTMCMC. \textbf{RF:} The random forest \cite{RF} implemented in \cite{scikit-learn}. \textbf{XGBoost:} the XGBoost \cite{XGBoost}. \textbf{LightGBM:} the light GBM implemented in \cite{lightGBM}. \textbf{BART100(50) etc.:} the BART \cite{BART} implemented in \cite{BART_R}. The number of trees in the BART model is assumed to be one for comparison with our method under the same condition. It can be specified by \texttt{ntree} option. The number at the end has a similar meaning to MTMCMC.

\subsection{Experiment 1: Bayes Optimality of Prediction}

\textbf{Purpose:} we confirm the Bayes optimality of our prediction method. 
Under the Bayes criterion, our method is expected to outperform any other methods for synthetic data generated from the assumed stochastic model. In particular, our method cannot be outperformed by any function tree based methods such as RF, XGBoost, and LightGBM. Our method will also outperform MTRF because it approximates \eqref{tildetildeq}.

\textbf{Conditions:} we assume $p=0$ and $q=20$. Therefore, all the explanatory variables are binary. $\mathcal{Y}$ is also the binary set $\{ 0, 1\}$. We assume $D_\mathrm{max}=10$. $p(\bm k)$ is the uniform distribution on $\mathcal{K}$. $p(T)$ is the tree distribution of \eqref{DistributionTree} with $g_s = 0.75$ for any $s \in \mathcal{S}_\mathrm{max}$. We generate $\bm k$ and $T$ 100 times. Subsequently, we generate $\bm \theta$, $\bm x^n$, and $y^n$ 100 times for each $\bm k$ and $T$. $\theta_s$ is independently distributed with $p(\theta_s) = \mathrm{Beta}(\theta_s | 0.5, 0.5)$. The $i$th explanatory variable $\bm x_i$ is independently generated from the uniform distribution on $\{ 0, 1\}^q$. The data observation model $p(y | \theta_s)$ is the Bernoulli distribution $\mathrm{Bern}(y | \theta_s)$. Each method is trained with the generated data up to the size of 200. The size of test data is 100. The other conditions are given in the supplementary materials.

\begin{wrapfigure}[15]{r}[0pt]{0.5\textwidth}
\centering
\includegraphics[width=\linewidth]{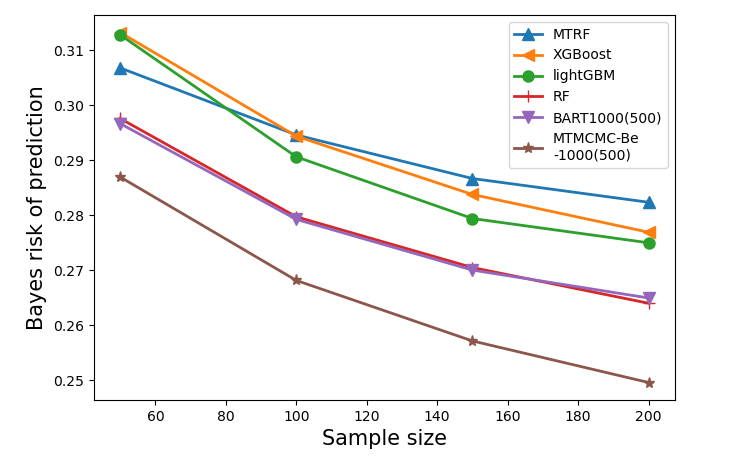}
\caption{The result of Experiment 1.}
\label{result_of_experiment_2}
\end{wrapfigure}

\textbf{Results:} Figure \ref{result_of_experiment_2} shows the approximated Bayes risk of the prediction, i.e., the average of the classification error ratio for all the generated models, parameters, and data. As expected, our proposed method outperforms any other methods.
Notably, our method outperformed the BART, which is also based on Bayesian inference. We consider it is because the model assumed in the BART is slightly different from our model. While the binary objective variable is directly output from a distribution on a leaf node in our model, it follows a logit-transformed distribution of a continuous output from a leaf node in the BART model.

\

\subsection{Experiment 2: Real-World Example}

\textbf{Purpose:} We confirm the performance of our method on real-world example. 

\textbf{Conditions:} We apply our method to a binary classification task on data about the Titanic \cite{titanic_origin}\footnote{Data obtained from http://hbiostat.org/data courtesy of the Vanderbilt University Department of Biostatistics.} and a regression task on data about abalones from the UCI repository \cite{UCIrepo}. Note that $\mathcal{Y}$ of the abalone data is $\mathbb{Z}_{\geq 0}$.
We perform the five-fold cross-validation for both data. The other conditions are given in the supplementary materials. In this experiment, we use the REMC method to deal with multimodality of the poseterior distribution. For more detail, see the supplementary materials. Only in this experiment, we used the sum of tree models, i.e., BART model with a default \texttt{ntree} option, for comparison, although our method is based on a single model tree. It will be represented by BART-Multi in figures.

\begin{figure}[h]
\centering
\begin{minipage}[t]{0.47\textwidth}
\includegraphics[width=\linewidth]{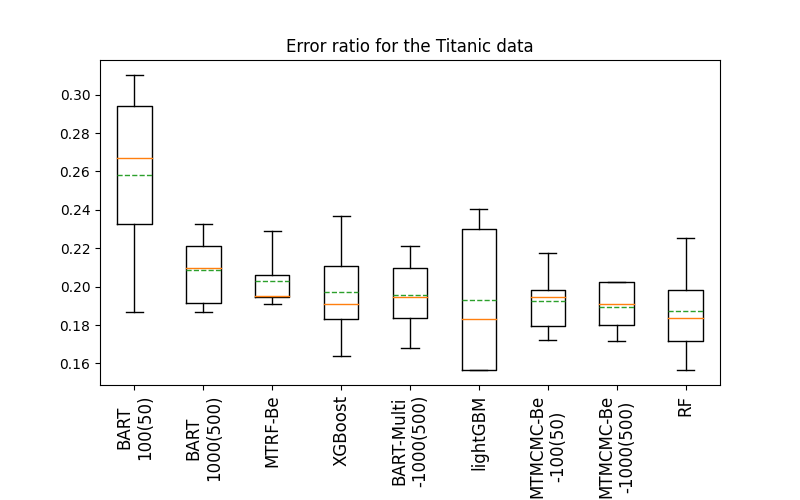}
\caption{The prediction error ratio for survival of passengers on the Titanic \cite{titanic_origin} (in order of the average error ratio).}
\label{result_of_experiment_3_revised}
\end{minipage}
\hfill
\begin{minipage}[t]{0.47\textwidth}
\includegraphics[width=\linewidth]{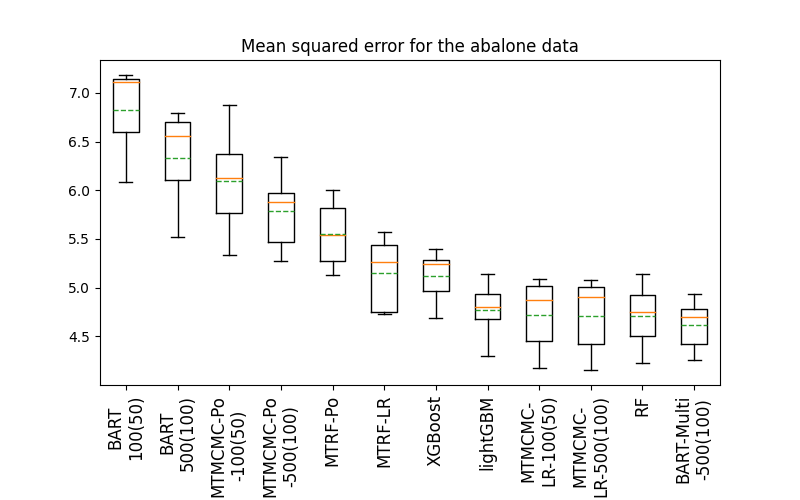}
\caption{The mean squared error for the abalone ages \cite{UCIrepo} (in order of the average of the mean squared error).}
\label{result_of_experiment_abalone_revised}
\end{minipage}
\end{figure}

\textbf{Results:} Figures \ref{result_of_experiment_3_revised} and \ref{result_of_experiment_abalone_revised} show the box plot of the prediction error ratio for each validation data of the Titanic and the mean squared error for each validation data of the abalones, respectively. On average, our method showed comparable performance with state-of-the-art methods such as XGBoost and LightGBM.
Although the sum of tree models (BART-Multi) showed the best performance in Fig.\ \ref{result_of_experiment_abalone_revised}, the gap between our method and the sum of tree models can be decreased by extending our model to a sum of meta-tree models. 

Another interesting insight from these results would be the difference that comes from the number of MCMC iterations. For BART, decreasing the number of iterations, the error ratio and the mean squared error were increased. In contrast, those of our method were not so increased. This indicates the efficiency of our sampling method compared with that of BART.

Regarding the computational cost of our method, using Python on a normal labtop, each MCMC iteration on a single chain requires approximately 46 msec for abalone data, where $K=2$, $p=8$, $q=3$, $D_\mathrm{max}=10$, and LR models are assumed on the leaf nodes (the most complex setting in this experiment). For more detail, see the supplementary materials. Note that our motivation for the Bayes optimal prediction is to solve the overfitting for small data, and scalability is less important.

In summary, as shown in the above result, our model flexibly expresses the data observation processes by assuming different models on the leaf nodes. It supports both categorical and continuous objective variables with categorical, continuous, or mixed explanatory variables. Moreover, we can prevent overfitting because their parameters can be learned Bayes optimally. Further, the computational cost is not so expensive. Therefore, we believe that our method should be at least a possible choice for real-world tasks.




\clearpage
\appendix

\section{Proof of Proposition \ref{OptimalPrediction}}
From Bayes' theorem, we have
\begin{align}
\mathrm{BR}  (\delta) = \int_{\mathcal{Y}^n} \Big( \int_{\mathcal{Y}} p(y_{n+1} | \bm x_{n+1}, \bm x^n, y^n) \ell_{0-1}(\delta (\bm x_{n+1}, \bm x^n, y^n), y_{n+1}) \mathrm{d}y_{n+1} \Big) p(y^n | \bm x^n) \mathrm{d}y^n, \label{ModifiedBayesRisk}
\end{align}
where
\begin{align}
p(y_{n+1} | \bm x_{n+1}, \bm x^n, y^n) \coloneqq \sum_{\bm k \in \mathcal{K}} \sum_{T \in \mathcal{T}} \int_{\bm \Theta} p(y_{n+1} | \bm x_{n+1}, \bm \theta, T, \bm k) p(\bm \theta, T, \bm k | \bm x^n, y^n) \mathrm{d} \bm \theta. \label{PredictiveDistribution}
\end{align}
If we find the decision function $\delta$ that minimizes the brackets in \eqref{ModifiedBayesRisk}, then this is the optimal decision $\delta^*$, and it is easy to see that
\begin{align}
\delta^* & (\bm x_{n+1}, \bm x^n, y^n)=\argmax_{y_{n+1} \in \mathcal{Y}} p(y_{n+1} | \bm x_{n+1}, \bm x^n, y^n). \label{Optimal_a}
\end{align}

Hence, from \eqref{PredictiveDistribution} and \eqref{Optimal_a}, we complete the proof of Proposition \ref{OptimalPrediction}.

\section{Proof of Proposition \ref{MCMCValidity}}

Herein, we prove Proposition \ref{MCMCValidity} in general, i.e., we consider a general process to obtain an MCMC sample $\{ z^{(t)} \}_{t=1}^{t_\mathrm{end}}$ from an objective distribution $p(z)$ by the MH algorithm. (In our case, $z$ and $p(z)$ should be replaced with $\bm k$ and $p(\bm k | \bm x^n, y^n)$, respectively.) In the $t$th iteration of the MH algorithm, $z^*$ is generated from a proposal distribution $q(z^* | z^{(t-1)})$. Then, it will be accepted according to the following acceptance probability
\begin{align}
A (z^*, z^{(t-1)}) \coloneqq \min \left\{ 1, \frac{p(z^*) q(z^{(t-1)} | z^*)}{p(z^{(t-1)}) q(z^* | z^{(t-1)})} \right\}.
\end{align}
If $z^*$ is accepted, we make $z^{(t)} \leftarrow z^*$, otherwise $z^{(t)} \leftarrow z^{(t-1)}$. 

According to the above process, the transition probability $T(z^{(t)} | z^{(t-1)})$ is represented as follows.
\begin{align}
    T(z^{(t)} | z^{(t-1)}) = \begin{cases}
        q(z^{(t)} | z^{(t-1)}) A(z^{(t)}, z^{(t-1)}), & \! z^{(t)} \neq z^{(t-1)}\\
        1 \!-\!  \sum_{z \neq z^{(t)}} q(z | z^{(t-1)}) A(z, z^{(t-1)}), & \! z^{(t)} = z^{(t-1)}\\
    \end{cases}
\end{align}
Therefore, if $q(z^* | z^{(t-1)})$ is time-invariant\footnote{For $t \neq t'$, if $z^{(t-1)} = z^{(t'-1)}$ holds, then $q(z^* | z^{(t-1)}) = q(z^* | z^{(t'-1)})$ holds for any $z^*$.}, then $T(z^{(t)} | z^{(t-1)})$ is also time-invariant. Moreover, if $q(z^* | z^{(t-1)}) > 0$ holds for any $z^*$ and $z^{(t-1)}$, the Markov chain of $z^{(t)}$ is ergodic.

It is known that any time-invariant and ergodic Markov chain has a unique stationary distribution $p^*(z)$. It is also known that if any distribution $\tilde{p}(z)$ satisfies the following condition called detailed balance,
\begin{align}
    \tilde{p}(z) T(z'|z) = \tilde{p}(z') T(z|z') \quad \text{for any } z \text{ and } z', \label{DetailedBalance}
\end{align}
then $\tilde{p}(z)$ is the stationary distribution, i.e., $\tilde{p}(z) = p^*(z)$. (It is a sufficient condition but not a necessary condition.)

Then, we prove the objective distribution $p(z)$ is the stationary distribution of the Markov chain whose transition probability is $T(z^{(t)}|z^{(t-1)})$ by showing the objective distribution satisfies the detailed balance. If $z = z'$, the equation in \eqref{DetailedBalance} clearly holds. If $z \neq z'$, we have
\begin{align}
    p(z) T(z'|z) &= p(z) q(z'|z) A(z',z) \\
    &= p(z) q(z'|z)  \min \left\{ 1, \frac{p(z') q(z | z')}{p(z) q(z' | z)} \right\} \\
    &= \min \{ p(z) q(z'|z), p(z') q(z | z') \} \\
    &= p(z') q(z | z') \min \left\{ \frac{p(z) q(z'|z)}{p(z') q(z | z')}, 1 \right\} \\
    &= p(z') q(z | z') A(z,z') \\
    &= p(z') T(z|z').
\end{align}
Therefore, the objective distribution $p(z)$ is the stationary distribution of Markov chain induced from the aforementioned process. Consequently, the empirical distribution $\{z^{(t)}\}_{t=1,2,\dots}$ converges to the objective distribution $p(z)$ after sufficient iteration. 

\section{Tuning Algorithm of Additional Parameter of Meta-Tree Markov Chain Monte Carlo Methods}

The additional parameter $\bar{g}$ in \eqref{Probability of T^(t)} should be tuned in the burn-in phase. We
did it by Algorithm \ref{tuning_g_bar}. In our experiments, we set $r_\mathrm{obj} = 0.3$, $\rho = 0.99$, and $\phi = 0.999$.
\begin{algorithm}[htb]
\caption{Tuning algorithm of $\bar{g}$}
\label{tuning_g_bar}
\begin{algorithmic}[1]
\REQUIRE $r_\mathrm{obj} \in [0,1]$, $\rho \in [0,1]$, $\phi \in [0,1]$
\ENSURE $\bar{g} \in [0,1]$
\STATE $\bar{g} \leftarrow 0$
\STATE $N_\mathrm{accept} \leftarrow 1$
\STATE $N_\mathrm{propose} \leftarrow 1$
\STATE $N'_\mathrm{propose} \leftarrow 1$
\WHILE {Burn-in phase}
    \STATE Propose $\bm k^*$
    \IF {$\bm k^*$ is accepted}
        \STATE $N_\mathrm{accept} \leftarrow \rho N_\mathrm{accept} + 1$
    \ELSE
        \STATE $N_\mathrm{accept} \leftarrow \rho N_\mathrm{accept}$
    \ENDIF
    \STATE $N_\mathrm{propose} \leftarrow \rho N_\mathrm{propose} + 1$
    \STATE $\hat{r} \leftarrow N_\mathrm{accept} / N_\mathrm{propose}$
    \IF {$\hat{r} > r_\mathrm{obj}$}
        \STATE $\hat{g}_\mathrm{tmp} \leftarrow \hat{g} \cdot r_\mathrm{obj} / \hat{r}$
    \ELSE
        \STATE $\hat{g}_\mathrm{tmp} \leftarrow 1 - (1-\hat{g}) (1-r_\mathrm{obj}) / (1-\hat{r})$
    \ENDIF
    \STATE $N'_\mathrm{propose} \leftarrow \phi N'_\mathrm{propose} + 1$
    \STATE $\hat{g} \leftarrow (\phi \hat{g} + \hat{g}_\mathrm{tmp}) / N'_\mathrm{propose}$
\ENDWHILE
\STATE \textbf{return} $\hat{g}$
\end{algorithmic}
\end{algorithm}

\section{Other Examples of Proposal Distributions}

We show other examples of the proposal distributions of $\bm k^*$. Their effectiveness will be numerially compared in the next section.

\subsection{Uniform Proposal Distribution}

For comparison, we utilize the uniform distribution on $\mathcal{K}$ as the proposal distribution $q(\bm k^* | \bm k^{(t-1)}) = (p+q)^{-|\mathcal{I}_\mathrm{max}|}$. For this type of proposal distribution, the acceptance probability that satisfies the detailed balance is derived as follows:
\begin{align}
A(\bm k^*, \bm k^{(t-1)}) = \min \left\{ 1, \frac{p(y^n | \bm x^n, \bm k^*)}{p(y^n | \bm x^n, \bm k^{(t-1)})} \right\}. \label{AcceptanceForUniform}
\end{align}

\subsection{Tree Prior Based Proposal Distribution}

We can utilize the tree prior \eqref{DistributionTree} to generate $\tilde{T}$ instead of \eqref{Probability of T^(t)}. Then, the proposal distribution $q(\bm k^* | \bm k^{(t-1)})$ is represented as follows:
\begin{align}
q(\bm k^* | \bm k^{(t-1)}) = p(\tilde{T}) (p+q-1)^{-|\mathcal{L}_{\tilde{T}} \cap \mathcal{I}_\mathrm{max}|} (p+q)^{-|\mathcal{I}_\mathrm{max} \backslash \mathcal{S}_{\tilde{T}}|}. \label{PriorProposal}
\end{align}
The acceptance probability that satisfies the detailed balance for this proposal distribution is the same as \eqref{AcceptanceForUniform}.

\subsection{Other Examples of Tree Posterior Based Proposal Distribution}

In \eqref{Probability of T^(t)}, we truncated the hyperparameter $g_{s|\bm x^n, y^n, \bm k^{(t-1)}}$ by $\bar{g}$ to ensure the ergodicity and induce a jump. We also utilize a reduced one, such as,
\begin{align}
q(\tilde{T} | \bm x^n, y^n, \bm k^{(t-1)}) = \prod_{s \in \mathcal{I}_{\tilde{T}}} \alpha g_{s|\bm x^n, y^n, \bm k^{(t-1)}} \prod_{s' \in \mathcal{L}_{\tilde{T}}} (1-\alpha g_{s'|\bm x^n, y^n, \bm k^{(t-1)}}), \label{ProposalReduce}
\end{align}
where $\alpha$ is in the range of $[0, 1]$.

Further, not only reducing the large $g_{s|\bm x^n, y^n, \bm k^{(t-1)}}$, we can also amplify the small $g_{s|\bm x^n, y^n, \bm k^{(t-1)}}$ as follows.
\begin{align}
&q(\tilde{T} | \bm x^n, y^n, \bm k^{(t-1)}) \nonumber \\
&= \prod_{s \in \mathcal{I}_{\tilde{T}}} \bigl( (g_s + \alpha (g_{s|\bm x^n, y^n, \bm k^{(t-1)}} - g_s) \bigr) \prod_{s' \in \mathcal{L}_{\tilde{T}}} \Bigl(1- \bigl(g_{s'} + \alpha (g_{s'|\bm x^n, y^n, \bm k^{(t-1)}} - g_{s'}) \bigr) \Bigr), \label{ProposalAmplify}
\end{align}
where $g_s$ is the hyperparameter of the prior \eqref{DistributionTree}.

For \eqref{ProposalReduce} and \eqref{ProposalAmplify}, the acceptance probability is defined in a similar manner to \eqref{acceptance}. Many terms in the numerator and the denominator of \eqref{acceptance} are canceled in a similar manner to Remark \ref{cancel}. We can tune $\alpha$ in the same algorithm as Algorithm \ref{tuning_g_bar}.

\section{Computationally Efficient Proposal Distribution}

As described in Remark \ref{cancel}, we can efficiently evaluate the acceptance probability \eqref{acceptance} by calculating $q(\tilde{T} | \bm x^n, y^n, \bm k^{(t-1)})$ and $q(\tilde{T} | \bm x^n, y^n, \bm k^*)$, which can be obtained from $g_{s|\bm x^n, y^n, \bm k^{(t-1)}}$ and $g_{s|\bm x^n, y^n, \bm k^*}$ only for $s \in \mathcal{S}_{\tilde{T}}$. Therefore, the computational cost to evaluate the acceptance probability is $O(|\mathcal{S}_{\tilde{T}}|)$. However, to sample $\bm k^*$, we have to remember $k_s^{(t-1)}$ for all the node $s \in \mathcal{S}_{T_\mathrm{max}}$. It is because $s \in \mathcal{L}_{\tilde{T}}$ holds with non-zero probability for all $s \in \mathcal{S}_{T_\mathrm{max}}$, and if $s \in \mathcal{L}_{\tilde{T}}$ holds, then $k_s^*$ must be different from $k_s^{(t-1)}$. In this section, we describe a method to reduce this complexity. Although the explanation is based on the proposal distribution \eqref{TruncatedProposal} in the main article, this method is also applied for other proposal distributions based on \eqref{ProposalReduce} or \eqref{ProposalAmplify}.

First, let $T_{\bm x^n, \bm k}$ denote the minimal tree that contains the paths from the root node $s_\lambda$ to the leaf node $s_{\bm k, T_\mathrm{max}}(\bm x_i)$ for all $i \in \{ 1, 2, \dots , n \}$. In other words, $T_{\bm x^n, \bm k}$ is the minimal tree used during the observation of $y^n$ for given $\bm x^n$ and $\bm k$. Since the length of these paths is $D_\mathrm{max}$, $|\mathcal{S}_{T_{\bm x^n, \bm k}}| \leq n D_\mathrm{max}$ holds.

Next, we define the following parameter for all $s \in \mathcal{S}_{T_\mathrm{max}}$:
\begin{align}
    \tilde{g}_{s|\bm x^n, y^n, \bm k} \coloneqq
    \begin{cases}
        \min \{ g_{s|\bm x^n, y^n, \bm k}, \bar{g} \}, & s \in \mathcal{I}_{T_{\bm x^n, \bm k}}, \\
        0, & \mathrm{otherwise}.
    \end{cases}
\end{align}

Then, we generate $\tilde{T}$ according to the following distribution instead of \eqref{Probability of T^(t)}.
\begin{align}
    q(\tilde{T} | \bm x^n, y^n, \bm k^{(t-1)}) = \prod_{s \in \mathcal{I}_{\tilde{T}}} \tilde{g}_{s|\bm x^n, y^n, \bm k^{(t-1)}} \prod_{s' \in \mathcal{L}_{\tilde{T}}} (1-\tilde{g}_{s'|\bm x^n, y^n, \bm k^{(t-1)}}).
\end{align}

Lastly, we generate $\bm k^*$ as follows. For $s \in \mathcal{I}_{\tilde{T}}$, $k_s^{(t-1)}$ is fixed and $k_s^* = k_s^{(t-1)}$. For $s \in \mathcal{L}_{\tilde{T}} \cap \mathcal{I}_{T_{\bm x^n, \bm k^{(t-1)}}}$, $k_s^{(t-1)}$ is changed according to the uniform distribution on $\{ 1, 2, \dots , p+q \} \backslash \{ k_s^{(t-1)} \}$. For the other nods, $k_s^{(t-1)}$ is changed according to the uniform distribution on $\{ 1, 2, \dots, p+q \}$.

Therefore, we do not require $k_s^{(t-1)}$ on $s \notin \mathcal{I}_{\tilde{T}_{\bm x^n, \bm k^{(t-1)}}}$ to generate $\bm k^*$. Moreover, since $\tilde{T}$ is uniquely determined from $\bm k^{(t-1)}$ and $\bm k^*$, $q(\bm k^* | \bm k^{(t-1)})$ is represented as follows.
\begin{align}
    q(\bm k^* | \bm k^{(t-1)}) = q(\tilde{T} | \bm x^n, y^n, \bm k^{(t-1)}) (p+q-1)^{-\left| \mathcal{L}_{\tilde{T}} \cap \mathcal{I}_{T_{\bm x^n, \bm k^{(t-1)}}} \right|} (p+q)^{\left|\mathcal{I}_\mathrm{max} \backslash \left(\mathcal{S}_{\tilde{T}} \cap \mathcal{I}_{T_{\bm x^n, \bm k^{(t-1)}}}\right) \right|}. \label{lean_proposal_distribution}
\end{align}

Further, $\mathcal{I}_{\tilde{T}} \subset \mathcal{I}_{T_{\bm x^n, \bm k^*}}$ holds for any $\bm k^*$ generated through $\tilde{T}$. Similarly, $(\mathcal{L}_{\tilde{T}} \cap \mathcal{I}_{T_{\bm x^n, \bm k^{(t-1)}}}) = (\mathcal{L}_{\tilde{T}} \cap \mathcal{I}_{T_{\bm x^n, \bm k^*}})$ and $(\mathcal{S}_{\tilde{T}} \cap \mathcal{I}_{T_{\bm x^n, \bm k^{(t-1)}}}) = (\mathcal{S}_{\tilde{T}} \cap \mathcal{I}_{T_{\bm x^n, \bm k^*}})$ hold. Therefore, the transition from $\bm k^*$ to $\bm k^{(t-1)}$ cannot occur through any tree other than $\tilde{T}$, and we can cancel the denominator and the numerator of the acceptance probability in a similar manner to Remark \ref{cancel}.

In summary, by using the proposal distribution in \eqref{lean_proposal_distribution}, we can sample $\bm k^*$ without using $k_s^{(t-1)}$ on $s \notin \mathcal{I}_{T_{\bm x^n, \bm k^{(t-1)}}}$ and we can evaluate $q(\tilde{T} | \bm x^n, y^n, \bm k^{(t-1)})$ and $q(\tilde{T} | \bm x^n, y^n, \bm k^*)$ using only the parameters on the nodes in $\mathcal{S}_{\tilde{T}}$.

\section{Complexity Analysis of Meta-Tree Markov Chain Monte Carlo Methods}

In this section, we summarize the computational complexity of MTMCMC methods. 
First, let $y_{s|\bm k}$ denote the set of objective variables of data points that pass through $s$ in the data generating process for given $\bm k$ and $T_\mathrm{max}$. Therefore, $\bigcup_{s \in \mathcal{L}(T)} y_{s|\bm k} = y^n$ holds for any $T$ in the meta-tree $M_{T_\mathrm{max}, \bm k}$. In each iteration of the MTMCMC methods, we have to do the following procedure.
\begin{enumerate}
    \item Calculate $\int p(y_{s|\bm k^{(t-1)}} | \theta_s) p(\theta_s) \mathrm{d}\theta_s$ for all $s \in \mathcal{S}_{T_{\bm x^n, \bm k^{(t-1)}}}$.
    \item Calculate $p(T | \bm x^n, y^n, \bm k^{(t-1)})$.
    \item Generate $\bm k^*$ according to $q(\bm k^* | \bm k^{(t-1)})$.
    \item Evaluate $A(\bm k^*, \bm k^{(t-1)})$.
\end{enumerate}
After $t_\mathrm{end}$ iterations, we calculate \eqref{ApproximatedBayesOptimal}. In the following, we evaluate the computational complexity of these procedures.

\textbf{Calculation of $\int p(y_{s|\bm k^{(t-1)}} | \theta_s) p(\theta_s) \mathrm{d}\theta_s$ for all $s \in \mathcal{S}_{T_{\bm x^n, \bm k^{(t-1)}}}$:} we consider the worst case where $T_{\bm x^n, \bm k^{(t-1)}} = T_\mathrm{max}$. For each node $s \in \mathcal{S}_{T_\mathrm{max}}$, the computational cost to calculate $\int p(y_{s|\bm k^{(t-1)}} | \theta_s) p(\theta_s) \mathrm{d}\theta_s$ is usually proportional to the number of data points when $p(y_{s|\bm k^{(t-1)}}|\theta_s)$ is an usual exponential family distribution. Although the number of data points assigned to each node $s$ depends on $\bm k^{(t-1)}$, the sum of the number of data points assigned to all the nodes at each depth is always $n$. Therefore, the computational cost to calculate $\int p(y_{s|\bm k^{(t-1)}} | \theta_s) p(\theta_s) \mathrm{d}\theta_s$ for all $s \in \mathcal{S}_{T_\mathrm{max}}$ is $O(n D_\mathrm{max})$. We can calculate $p(\theta_s | y_{s|\bm k^{(t-1)}})$ simultaneously.

\textbf{Calculation of $p(T | \bm x^n, y^n, \bm k^{(t-1)})$:} using the method in \cite{batch_updating}, we can calculate $p(T | \bm x^n, y^n, \bm k^{(t-1)})$ with a complexity of  $O(|\mathcal{S}_{T_{\bm x^n, \bm k^{(t-1)}}}|)$. Note that $|\mathcal{S}_{T_{\bm x^n, \bm k^{(t-1)}}}| \leq n D_\mathrm{max}$ always holds. We can calculate $p(y^n | \bm x^n, \bm k^{(t-1)})$ simultaneously.

\textbf{Generation of $\bm k^*$ according to $q(\bm k^* | \bm k^{(t-1)})$:} using the proposal distribution described in the previous section, we need not generate $k_s^*$ for $s \notin \mathcal{S}_{T_{\bm x^n, \bm k^*}}$. Therefore, the computational complexity is $O(|\mathcal{S}_{T_{\bm x^n, \bm k^*}}|)$.

\textbf{Evaluation of $A(\bm k^*, \bm k^{(t-1)})$:} to evaluate $A(\bm k^*, \bm k^{(t-1)})$, we have to calculate $p(y^n | \bm x^n, \bm k^{(t-1)})$, $p(y^n | \bm x^n, \bm k^*)$, $q(\tilde{T}|\bm x^n, y^n, \bm k^{(t-1)})$, and $q(\tilde{T}|\bm x^n, y^n, \bm k^*)$. $p(y^n | \bm x^n, \bm k^{(t-1)})$ is already calculated. $p(y^n | \bm x^n, \bm k^*)$ can be calculated in a similar manner to $p(y^n | \bm x^n, \bm k^{(t-1)})$ by using the method in \cite{batch_updating}. Its computational complexity is $O(|\mathcal{S}_{T_{\bm x^n, \bm k^*}}|)$. The computational complexity to calculate $q(\tilde{T}|\bm x^n, y^n, \bm k^{(t-1)})$ is $O(|\mathcal{S}_{\tilde{T}}|)$ as described in the previous section. When using the proposal distribution described in the previous section, $O(|\mathcal{S}_{\tilde{T}}|) = O(|\mathcal{S}_{T_{\bm x^n, \bm k^{(t-1)}}}|)$. The computational cost to calculated $q(\tilde{T}|\bm x^n, y^n, \bm k^*)$ is similarly evaluated.

\textbf{Calculation of \eqref{ApproximatedBayesOptimal}:} at this point, $p(\theta_s | y_{s|\bm k^{(t)}})$ and $p(T | \bm x^n, y^n, \bm k^{(t)})$ are already calculated for all $t \in \{ 1, 2, \dots , t_\mathrm{end} \}$. Therefore, by using the method in \cite{suko_alg,MTRF}, we can calculate $\tilde{q}(y_{n+1} | \bm x_{n+1}, \bm x^n, y^n, \bm k^{(t)})$, i.e., \eqref{q} and \eqref{tildeq}, with a complexity of $O(D_\mathrm{max})$. To calculate \eqref{ApproximatedBayesOptimal}, we have to take summation of them for $t \in \{ 1, 2, \dots , t_\mathrm{end} \}$. Therefore, the complexity is $O(t_\mathrm{end} D_\mathrm{max})$

Consequently, the total complexity of the MTMCMC method is roughly $O(t_\mathrm{end} n D_\mathrm{max})$.

\section{Extension to Replica Exchange Monte Carlo Methods}

In this section, we extend our MH method to REMC methods (e.g., \cite{ReplicaExchange}) to deal with multimodality of the posterior distribution. First, we define the following joint distribution over $\mathcal{K}^J$ for $J \in \mathbb{N}$.
\begin{align}
    q(\bm k_1, \bm k_2, \dots , \bm k_J) \coloneqq \prod_{j=1}^J q_j(\bm k_j) \coloneqq \prod_{j=1}^J \frac{p(\bm k_j | \bm x^n, y^n)^{\beta_j}}{\sum_{\bm k_j \in \mathcal{K}} p(\bm k_j | \bm x^n, y^n)^{\beta_j}},
\end{align}
where $0 \leq \beta_1 < \beta_2 < \cdots < \beta_J = 1$. Since $\beta_J = 1$, the marginal distribution $q_J(\bm k_J)$ is equivalent to the posterior distribution $p(\bm k | \bm x^n, y^n)$ required to calculate the Bayes optimal prediction. Therefore, we construct an MCMC method for this joint distribution $q(\bm k_1, \bm k_2, \dots , \bm k_J)$ and use the sample for only $\bm k_J$, ignoring those for $\bm k_1, \bm k_2, \dots , \bm k_{J-1}$.  In REMC methods, the sample from the joint distribution is obtained as follows:

\textbf{1.} For each $q(\bm k_j)$, run the MH method and obtain the sample $\bm k_j^{(1)}, \bm k_j^{(2)}, \dots$. The proposal distribution and the acceptance probability are similar to those in the usual MH method described in the main article.

\textbf{2.} Let $m \in \mathbb{N}$ be a predetermined number. For every $m$ iterations of the MH method, we randomly choose $j \in \{ 0, 1, \dots , J-1 \}$ and exchange $\bm k_j^{(t)}$ and $\bm k_{j+1}^{(t)}$ with probability
\begin{align}
    \frac{q_j(\bm k_{j+1}^{(t)})q_{j+1}(\bm k_j^{(t)})}{q_j(\bm k_j^{(t)})q_{j+1}(\bm k_{j+1}^{(t)})} = \frac{p(y^n | \bm x^n, \bm k_{j+1}^{(t)})^{\beta_j} p(y^n | \bm x^n, \bm k_j^{(t)})^{\beta_{j+1}}}{p(y^n | \bm x^n, \bm k_{j+1}^{(t)})^{\beta_{j+1}} p(y^n | \bm x^n, \bm k_j^{(t)})^{\beta_j}}, \label{RE_probability}
\end{align}
where we used the Bayes' theorem and Assumption \ref{Distribution_k}. (This procedure can be applied multiple times at the same $t$th iteration of the MH method.)

It is known that the above procedure satisfies the detailed balance condition and the obtained sample asymptotically follows $q(\bm k_1, \bm k_2, \dots , \bm k_J)$ after sufficient iterations. Since $\beta_j$ is monotonically increasing, the effect of multimodality of $p(\bm k | \bm x^n, y^n)$ is reduced for small $j$. Therefore, $\bm k_j^{(t)}$ for small $j$ tends to move over the multiple modes. Exchange these sample with probability \eqref{RE_probability}, the REMC methods ensure the detailed balance and provide the sample from the multiple modes.

\section{Detailed Conditions of Experiments in the Main Article}

\subsection{Detailed Condition of Experiment 1}

\textbf{MTMCMC:} $\bar{g}$ in \eqref{Probability of T^(t)} was adaptively tuned in the burn-in phase by the algorithm described in this supplementary material. The other hyperparameters, e.g., $D_\mathrm{max}$, $g_s$, etc., were the same as those used to generate the true model and data.

\textbf{MTRF:} The number of meta-trees used for the prediction was 100, which was the default value of the library \cite{bayesml}. The other hyperparameters, e.g., $D_\mathrm{max}$, $g_s$, etc., were the same as those used to generate the true model and data.

\textbf{RF:} The maximum depth was 10. The other hyperparameters were default values of the library \cite{scikit-learn}.

\textbf{XGBoost:} All the hyperparameters were default values of the library \cite{XGBoost}.

\textbf{LightGBM:} All the hyperparameters were default values of the library \cite{lightGBM}.

\textbf{BART:} We used the \textit{lbart} (logit BART) function implemented in \cite{BART_R}. The \texttt{ntree} option was set at 1 because the true model was represented by a single tree. The other hyperparameters were default values of the library.

\subsection{Detailed Condition of Experiment 2}

\subsubsection{Conditions for Classification}

\textbf{Data set:} The data set was about the sinking of the Titanic \cite{titanic_origin}. Each data point $(\bm x_i, y_i)$ was a pair of the information about the $i$th passenger and his or her survival. In other words, we performed a binary classification. The used explanatory variables are "pclass", "age", "sibsp", "parch", "fare", "sex", and "embarked". We encoded "sex" into 0 or 1, and "embarked" into 001, 010, and 100 (one-hot vectors). Then, the number of continuous features is $p=5$ and the number of categorical features is $q=4$. Missing values were filled with the mode of each variable. The sample size was 1309.

\textbf{MTMCMC:} We had $D_\mathrm{max} = 10$ and $g_s = 0.75$ for any $s \in \mathcal{S}_\mathrm{max}$. The distribution of $y$ assigned at each node $s$ and its prior distribution were assumed to be the Bernoulli distribution $\mathrm{Bern}(y|\theta_s)$ and the beta distribution $\mathrm{Beta}(\theta_s | 0.5, 0.5)$, respectively. $\bar{g}$ in \eqref{Probability of T^(t)} was fixed at 0.8. The number of replicas in the REMC method was 8. The replica exchange procedure was made every 10 iterations of the MH method. In each replica exchange procedure, randomly selected 4 replicas are sequentially tried to exchange.

\textbf{MTRF:} We had $D_\mathrm{max} = 10$ and $g_s = 0.75$ for any $s \in \mathcal{S}_\mathrm{max}$. The distribution of $y$ assigned at each node $s$ and its prior distribution were assumed to be the Bernoulli distribution $\mathrm{Bern}(y|\theta_s)$ and the beta distribution $\mathrm{Beta}(\theta_s | 0.5, 0.5)$, respectively. The number of meta-trees used for the prediction was 100, which was the default value of the library \cite{bayesml}.

\textbf{RF:} The maximum depth was 10. The other hyperparameters were default values of the library \cite{scikit-learn}.

\textbf{XGBoost:} All the hyperparameters were default values of the library \cite{XGBoost}.

\textbf{LightGBM:} The maximum depth of trees was fexed at 10. The other the hyperparameters were default values of the library \cite{lightGBM}. (This setting showed a better result than default maximum depth setting.)

\textbf{BART:} We used the \textit{lbart} (logit BART) function implemented in \cite{BART_R}. The \texttt{ntree} option was set at 1 for comparison with our method under the same condition. The other hyperparameters were default values of the library.

\subsubsection{Conditions for Regression}

\textbf{Data set:} 
The data set was about abalones from UCI repository \cite{UCIrepo}. Each data point $(\bm x_i, y_i)$ was a pair of physical measurements of abalones and its age. 
Therefore, the set of objective variable $\mathcal{Y}$ was $\mathbb{Z}_{\geq 0}$. We used all the explanatory variables. We encoded "Sex", which consists of "M", "F", and "I" (infant), into 001, 010, and 100 (one-hot vectors). Then, the number of continuous features is $p=7$ and the number of categorical features is $q=3$. (When we assume a linear regression model at each leaf node of model trees, we had $p=8$ because constant term was added.) The sample size was 4177.

\textbf{MTMCMC:} We had $D_\mathrm{max} = 10$ and $g_s = 0.75$ for any $s \in \mathcal{S}_\mathrm{max}$. In MTMCMC-Po, the distribution of $y$ assigned at each node $s$ and its prior distribution were assumed to be the Poisson distribution $\mathrm{Po}(y|\nu_s)$ and the gamma distribution $\mathrm{Gam}(\nu_s | 1, 1)$, respectively. In MTMCMC-LR, the distribution of $y$ assigned at each node $s$ and its prior distribution were assumed to be linear regression model $\mathcal{N}(y | \bm w_s^\top \bm x, \sigma_s^2)$ and the normal-gamma prior $\mathcal{N}(\bm w_s | \bm 0, \bm I / \sigma_s^2) \mathrm{Gam}(1 / \sigma_s^2 | 1, 1)$, respectively. $\bar{g}$ in \eqref{Probability of T^(t)} was tuned in the burn-in phase. The number of replicas in the REMC method was 8. The replica exchange procedure was made every 10 iterations of the MH method. In each replica exchange procedure, randomly selected 4 replicas are sequentially tried to exchange.

\textbf{MTRF:} We had $D_\mathrm{max} = 10$ and $g_s = 0.75$ for any $s \in \mathcal{S}_\mathrm{max}$. The number of meta-trees used for the prediction was 100, which was the default value of the library \cite{bayesml}. 

\textbf{RF:} The maximum depth was 10. The other hyperparameters were default values of the library \cite{scikit-learn}.

\textbf{XGBoost:} All the hyperparameters were default values of the library \cite{XGBoost}.

\textbf{LightGBM:} All the hyperparameters were default values of the library \cite{lightGBM}.

\textbf{BART:} We used the \textit{gbart} (generalized BART) function implemented in \cite{BART_R}. The \texttt{ntree} option was set at 1 for comparison with our method under the same condition. The other hyperparameters were default values of the library.

\textbf{BART-Multi:} BART with default \texttt{ntree} option.

\section{In-Depth Experiments on Algorithm Behavior}

\subsection{Experiment 3: Convergence to Exact Posterior}

\textbf{Purpose:} we confirm the convergence of the MCMC sample distribution to the exact posterior distribution. Since our proposal distributions satisfy the detailed balance and ergodicity, the approximated posteriors are expected to converge to the exact one. Further, we confirm the effectiveness of the design policy of the proposal distribution, compared with the uniform proposal distribution.

\textbf{Conditions:} to calculate the exact posterior distribution, we fix a true model with small $p$, $q$ and $D_\mathrm{max}$. Specifically, we perform the experiment under the following conditions. We assume $p=0$ and $q=5$. Therefore, all the explanatory variables are binary. $\mathcal{Y}$ is also the binary set $\{ 0, 1\}$. We assume $D_\mathrm{max} = 3$. Then, we have $|\mathcal{K}| = 5^7 = 78125$. Specific values of $\bm k$, $T$, and $\bm \theta$ are shown in Fig.\ \ref{true_model_in_experiment_1}. The data generative model $p(y | \theta_s)$ is the Bernoulli distribution $\mathrm{Bern}(y | \theta_s)$. The $i$th explanatory variable $\bm x_i$ is independently generated according to the uniform distribution on $\{ 0, 1\}^q$. Then, $y_i$ is generated from the model shown in Fig.\ \ref{true_model_in_experiment_1}. The sample size $n$ is 100 and the number of generated samples is 10.

\begin{figure}[tbp]
    \centering
    \includegraphics[width=0.4\linewidth]{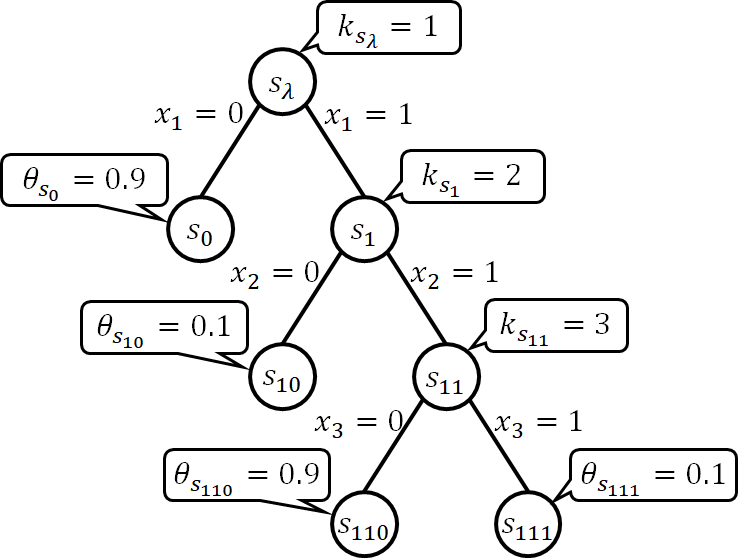}
    \caption{The true model assumed in the first experiment. Here, $\mathcal{Y} = \{ 0, 1\}$ and the parameters $\theta_s$ on the leaf nodes represent the probability that $y = 1$.}
    \label{true_model_in_experiment_1}
\end{figure}

For posterior learning, we independently assume the beta distribution $\mathrm{Beta}(\theta_s | 0.5, 0.5)$ as the prior distribution for each $\theta_s$. The hyperparameter of $p(T)$ is fixed at $g_s = 0.5$ for each $s \in \mathcal{I}_\mathrm{max}$. Herein, we utilize two proposal distributions: the uniform distribution and \eqref{TruncatedProposal}. The tuning parameter $\bar{g}$ in \eqref{Probability of T^(t)} of the tree posterior based proposal distribution $q(\bm k^* | \bm k^{(t-1)})$ is fixed at $0.75$. The burn-in length is 500 and the MCMC process is continued until 1000 samples are accepted.

\begin{figure}[tbp]
    \centering
    \includegraphics[width=0.5\linewidth]{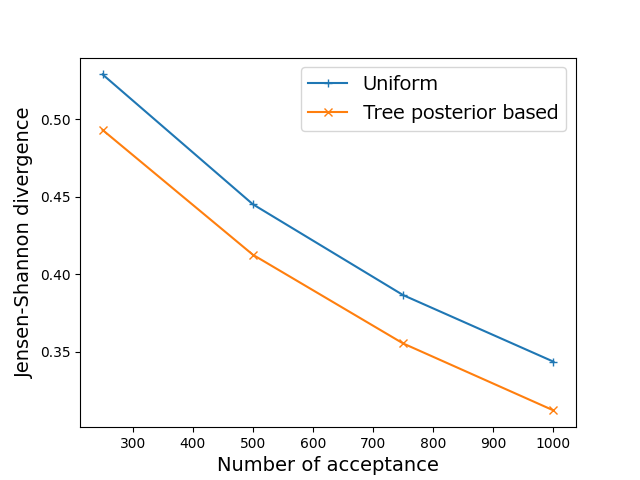}
    \caption{The Jensen-Shannon divergence between the exact posterior and the approximated posterior to the number of accepted tests of the MCMC.}
    \label{result_of_experiment_1}
\end{figure}

\textbf{Results:} we evaluate the distance $d(p, \hat{p})$ between the exact posterior distribution $p(\bm k | \bm x^n, y^n)$ and the approximated posterior distribution $\hat{p} (\bm k | \bm x^n, y^n)$ obtained from the MCMC sample by the following Jensen-Shannon divergence.\footnote{Since $\hat{p} (\bm k | \bm x^n, y^n)$ is a empirical distribution and takes 0 on some points in $\mathcal{K}$, the usual Kullback–Leibler divergence cannot be evaluated.}
\begin{align}
d(p, \hat{p}) \coloneqq \frac{1}{2} \sum_{\bm k \in \mathcal{K}} p(\bm k | \bm x^n, y^n) \log \frac{p(\bm k | \bm x^n, y^n)}{r(\bm k | \bm x^n, y^n)} + \frac{1}{2} \sum_{\bm k \in \mathcal{K}} \hat{p}(\bm k | \bm x^n, y^n) \log \frac{\hat{p}(\bm k | \bm x^n, y^n)}{r(\bm k | \bm x^n, y^n)},
\end{align}
where $r(\bm k | \bm x^n, y^n) \coloneqq (p(\bm k | \bm x^n, y^n) + \hat{p}(\bm k | \bm x^n, y^n))/ 2$ and we use the convention that $0 \log 0 = 0$.

\begin{figure*}
    \begin{subfigmatrix}{3}
    \subfigure[Model A]{\includegraphics[width=0.2\linewidth]{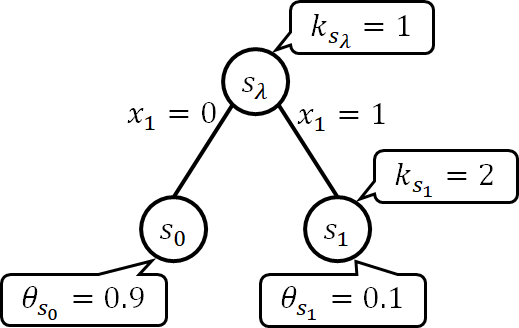}}
    \subfigure[Model B]{\includegraphics[width=0.25\linewidth]{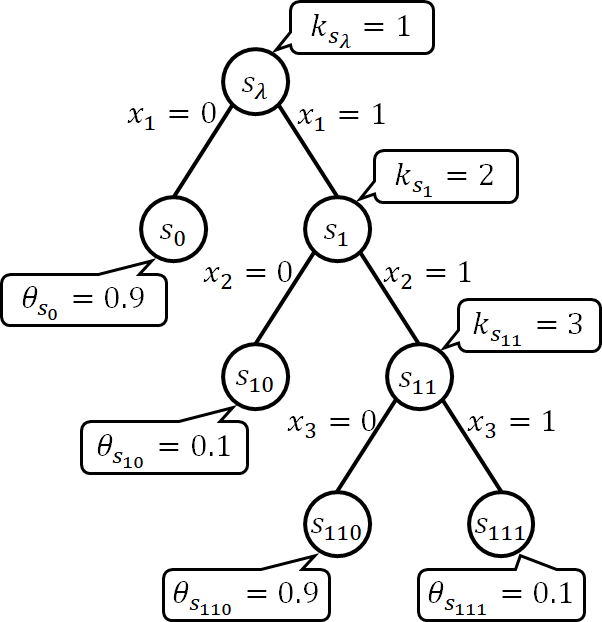}}
    \subfigure[Model C]{\includegraphics[width=0.5\linewidth]{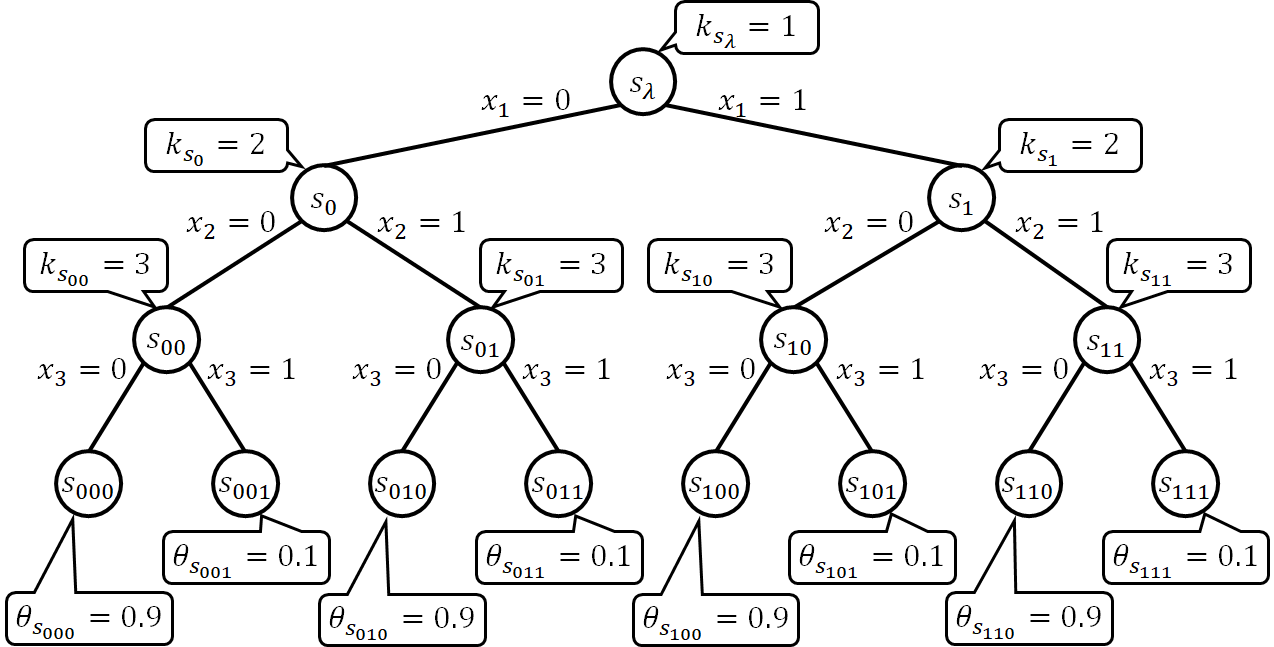}}
    \subfigure[Results for Model A]{\includegraphics[width=0.32\linewidth]{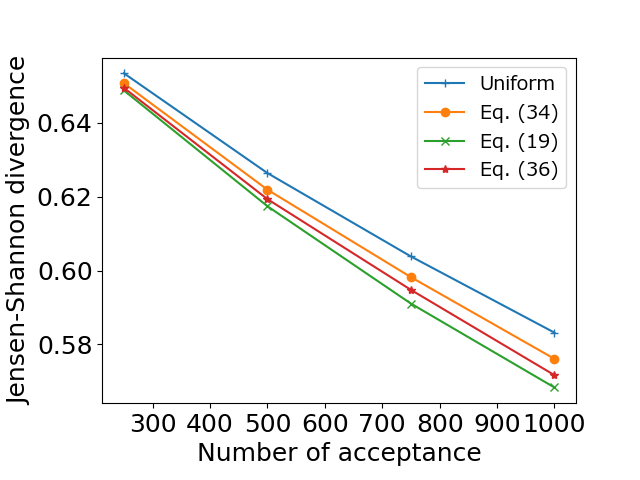}}
    \subfigure[Results for Model B]{\includegraphics[width=0.32\linewidth]{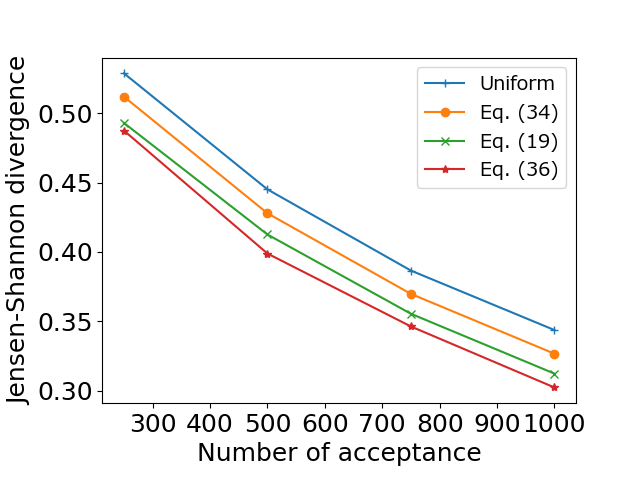}}
    \subfigure[Results for Model C]{\includegraphics[width=0.32\linewidth]{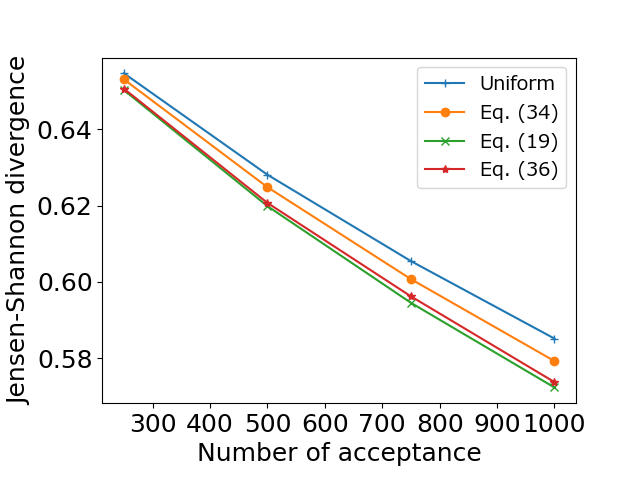}}
    \end{subfigmatrix}
    \caption{The models assumed in the experiment and the results of the Jensen-Shannon divergence for each proposal distributions.}
    \label{SupplementToExperiment1}
\end{figure*}

Figure \ref{result_of_experiment_1} shows the transition of the distance $d(p, \hat{p})$ for the increase of the number of the accepted tests. Both of the approximated posteriors converge to the exact one as expected. The convergence speed of the tree posterior based proposal distribution is faster than that of the uniform proposal distribution. In addition, the acceptance ratio of the tree posterior based proposal distribution was 0.274, while that of the uniform proposal distribution was 0.0118. These results support the effectiveness of our design policy of the proposal distribution. We also obtained similar results for other data generative models described in the next subsection.

\subsection{Experiment 4: Comparison of Convergence}

\begin{table}
    \centering
    \caption{Acceptance ratios for the compared proposal distributions.}
    \label{tab:acceptance_ratio}
    \begin{tabular}{l|cccc}
        \hline
        & \multicolumn{3}{c}{Acceptance ratio}\\
        $q(\bm k^* | \bm k^{(t-1)})$ & Model A & Model B & Model C \\
        \hline
        $(p+q)^{-|\mathcal{I}_\mathrm{max}|}$ & 0.11 & 0.012 & 0.11 \\
        Eq.\ \eqref{PriorProposal} & 0.27 & 0.073 & 0.29 \\
        Eq.\ \eqref{TruncatedProposal}  & 0.49 & 0.27 & 0.50 \\
        Eq.\ \eqref{ProposalAmplify} & 0.46 & 0.25 & 0.46 \\
        \hline
    \end{tabular}
\end{table}

We compare the convergence of the MCMC sample distribution obtained from the aforementioned four proposal distributions $q(\bm k^* | \bm k^{(t-1)})$: the uniform distribution, Eq.\ \eqref{PriorProposal}, Eq.\ \eqref{TruncatedProposal} and Eq.\ \eqref{ProposalAmplify}. We assumed three models shown in the upper side of Fig.\ \ref{SupplementToExperiment1}. Herein, we assumed $p=0$. The other hyperparameters are the same as those for Experiment 3. Resuls are shown in the lower side of Fig.\ \ref{SupplementToExperiment1} and Table \ref{tab:acceptance_ratio}. The tree posterior based proposal distributions \eqref{TruncatedProposal} and \eqref{ProposalAmplify} showed better performances, i.e., they showed faster convergence and higher acceptance ratio than the others. In particular, the uniform proposal distribution and Eq.\ \eqref{PriorProposal} showed extremely low acceptance ratio for Model B, which has an unbalanced shape.

\subsection{Experiment 5: Confirmation of Likelihood Behavior}

We confirm the behavior of our MCMC method from a perspective of likelihood. If our design policy of the proposal distribution $q(\bm k^* | \bm k^{(t-1)})$ works, the likelihood $p(y^n | \bm x^n, \bm k^{(t)})$ should increase in the early phase of the MCMC iterations and stay high.

Actually, we observed the desirable behavior for the real-world example \cite{titanic_origin} used in Experiment 2 in the main paper as shown in Fig.\ \ref{titanic_log_likelihood}.

\begin{figure}
\centering
\includegraphics[width=0.7\linewidth]{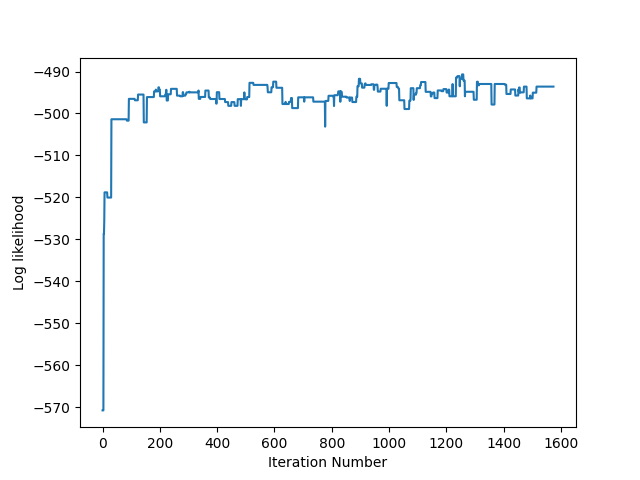}
\caption{Behavior of the log likelihood for the Titanic data \cite{titanic_origin}}
\label{titanic_log_likelihood}
\end{figure}

\section{Computing Resources}

The main computing resources used in our experiments are as follows:
\begin{itemize}
    \item For Experiment 1
    \begin{itemize}
        \item Desktop 1
        \begin{itemize}
            \item CPU:	Intel(R) Xeon(R) Gold 6128 CPU @ 3.40GHz
            \item Memory: 64GB
            \item OS: Windows 10 Pro
        \end{itemize}
        \item Desktop 2
        \begin{itemize}
            \item CPU:	Intel(R) Core(TM) i7-8700 CPU @ 3.20GHz
            \item Memory: 64GB
            \item OS: Windows 10 Pro
        \end{itemize}
    \end{itemize}
    \item For Experiment 2
    \begin{itemize}
        \item Laptop 1
        \begin{itemize}
            \item CPU:	Intel(R) Core(TM) i5-8265U CPU @ 1.60GHz
            \item Memory: 8GB
            \item OS: Windows 11 Pro
        \end{itemize}
    \end{itemize}
\end{itemize}

\end{document}